\documentclass[11pt]{article}

\usepackage{amsmath, amssymb,amsthm}
\usepackage{hyperref}
\usepackage{algorithm,algorithmic}
\usepackage{fullpage}
\usepackage{color}

\newtheorem{thm}{Theorem}
\newtheorem{lem}[thm]{Lemma}
\newtheorem{prop}[thm]{Proposition}
\newtheorem{cor}[thm]{Corollary}

\newtheorem{rem}{Remark}
\newtheorem*{thm*}{Theorem}
\newtheorem*{lem*}{Lemma}
\newtheorem*{prop*}{Proposition}
\newtheorem*{cor*}{Corollary}

\newcommand{\Tab}[1]{Table~\ref{#1}}

\newcommand{\Eqn}[1]{Eq.~(\ref{#1})}

\newcommand{\Lem}[1]{Lemma~\ref{#1}}
\newcommand{\Thm}[1]{Theorem~\ref{#1}}
\newcommand{\Prop}[1]{Proposition~\ref{#1}}

\renewcommand{\hat}{\widehat}
\renewcommand{\tilde}{\widetilde}

\renewcommand{\>}{{\,\rightarrow\,}}
\renewcommand{\=}{\stackrel{\triangle}{=}}

\newcommand{\hy}{{\widehat y}}

\newcommand{\argmax}{\operatorname{argmax}}

\newcommand{\R}{{\mathbb R}}
\newcommand{\Z}{{\mathbb Z}}
\newcommand{\N}{{\mathbb N}}
\renewcommand{\P}{{\mathbf P}}
\newcommand{\E}{{\mathbf E}}

\newcommand{\1}{{\mathbf 1}}

\newcommand{\cA}{{\mathcal A}}
\newcommand{\A}{{\mathbf A}}
\newcommand{\cB}{{\mathcal B}}
\newcommand{\B}{{\mathbf B}}

\newcommand{\CC}{{\mathcal C}}

\renewcommand{\H}{{\mathcal H}}

\newcommand{\M}{{\mathbf M}}

\newcommand{\cP}{{\mathcal P}}

\newcommand{\V}{{\mathcal V}}
\newcommand{\X}{{\mathcal X}}
\newcommand{\Y}{{\mathcal Y}}
\newcommand{\Pf}{{\mathcal P}}
\newcommand{\e}{{\mathbf e}}
\newcommand{\f}{{\mathbf f}}
\newcommand{\g}{{\mathbf g}}
\newcommand{\h}{{\mathbf h}}
\newcommand{\p}{{\mathbf p}}

\renewcommand{\u}{{\mathbf u}}
\renewcommand{\v}{{\mathbf v}}

\newcommand{\x}{{\mathbf x}}

\newcommand{\C}{{\mathbf C}}
\newcommand{\G}{{\mathbf G}}

\newcommand{\boldeta}{{\boldsymbol \eta}}

\newcommand{\bmu}{{\boldsymbol \mu}}

\newcommand{\reg}{\textup{\textrm{regret}}}

\newcommand{\conf}{\textup{\textrm{conf}}}

\newcommand{\CPE}{\textup{\textrm{CPE}}}

\newcommand{\FW}{\textup{\textrm{FW}}}

\newcommand{\apx}{\textup{\textrm{apx}}}
\newcommand{\TPR}{\textrm{{TPR}}}
\newcommand{\vol}{\textrm{{vol}}}
\newcommand{\MM}{\textrm{{MM}}}

\newcommand{\Prec}{\textrm{{Prec}}}

\newcommand{\term}{\textup{\textrm{term}}}

\allowdisplaybreaks

\begin{document}

\title{Consistent Classification Algorithms for\\Multi-class Non-Decomposable Performance Metrics}

\author{Harish G. Ramaswamy \hspace{2em} Harikrishna Narasimhan \hspace{2em} Shivani Agarwal\\
Department of Computer Science and Automation\\
Indian Institute of Science, Bangalore -- 560012, India\\
\texttt{\{harish\_gurup, harikrishna, shivani\}@csa.iisc.ernet.in}}

\maketitle


\begin{abstract}
We study consistency of learning algorithms for a multi-class performance metric that is a non-decomposable function of the confusion matrix of a classifier and cannot be expressed as a sum of losses on individual data points; examples of such performance metrics include the micro and macro F-measure used widely in information retrieval and the multi-class G-mean metric popular in class-imbalanced problems. While there has been much work in recent years in understanding the consistency properties of learning algorithms for `binary' non-decomposable metrics  \cite{Ye+12, Narasimhan+14, NatarajanRavi14, Parambath+14, Kotlowski14}, little is known either about the form of the optimal classifier for a general multi-class non-decomposable metric, or about how these learning algorithms generalize to the multi-class case. In this paper, we provide a unified framework for analyzing a multi-class non-decomposable performance metric, where the problem of finding the optimal classifier for the performance metric is viewed as an optimization problem over the space of all confusion matrices achievable under the given distribution. Using this framework, we show that (under a continuous distribution) the optimal classifier for a multi-class performance metric can be obtained as the solution of a cost-sensitive classification problem, thus generalizing several previous results on specific binary non-decomposable metrics. We then design a consistent learning algorithm for concave multi-class performance metrics that proceeds via a sequence of cost-sensitive classification problems, and can be seen as applying the conditional gradient (CG) optimization method over the space of feasible confusion matrices. To our knowledge, this is the first efficient learning algorithm (whose running time is polynomial in the number of classes) that is provably consistent for a large family of multi-class non-decomposable metrics. Our consistency result makes use of a novel proof technique based on the convergence analysis of the CG method. 
\end{abstract}


\section{Introduction}
\label{sec:intro}
In many real-world classification tasks, the performance metric used to evaluate a multi-class classifier is often a non-decomposable function of the confusion matrix of a classifier and cannot be expressed as a sum or expectation of losses on individual data points; this includes for example, the micro and macro F-measure used widely in information retrieval and the multi-class G-mean metric popular in class-imbalanced problems (see Table \ref{tab:perf-metrics} for other examples). While there has been much work in recent years in understanding the consistency properties of plug-in or cost-sensitive risk minimization based learning algorithms for `binary' non-decomposable metrics \cite{Ye+12, Narasimhan+14, NatarajanRavi14, Parambath+14, Kotlowski14}, little is known about the form of the optimal classifier for a general multi-class non-decomposable metric, or about how these learning algorithms for binary performance metrics, which make use of a brute-force line search of a single threshold/cost parameter, generalize to the multi-class case, where the number of parameters needed to be tuned scales with the number of classes. 

In this paper, we provide a general framework for analysing a multi-class non-decomposable performance metric, where the problem of finding optimal classifier for the performance metric is viewed as an optimization problem over the space of all confusion matrices achievable under the given distribution. Using this framework, we show that, under a continuous distribution, the optimal classifier for any multi-class performance metric (that satisfies a mild condition) can be obtained by solving a cost-sensitive classification problem, where the costs are given by the gradient of the non-decomposable metric at the optimal confusion matrix. This result generalizes a previous result for binary non-decomposable metrics \cite{Narasimhan+14} and also recovers several previous results on the form of the optimal classifier for specific binary performance metrics \cite{Lipton+14,NatarajanRavi14,Kotlowski14}. 

A natural first-cut learning algorithm that arises from the above characterization is one that learns a plug-in classifier by applying an empirical weight matrix chosen by a brute-force search to a suitable class probability estimator. While this method can be shown to be statistically consistent with respect to the given performance metric (under a continuous distribution), it becomes computationally inefficient when the number of classes is large. As an alternative, we provide an efficient learning algorithm based on the conditional gradient (CG) optimization method (which we call the `BayesCG' algorithm) that avoids a brute-force search over costs and can be seen as instead running the CG method over the space of feasible confusion matrices; the resulting algorithm proceeds via a sequence of cost-sensitive classification problems, the solutions for which take the form of plug-in classifiers. We show that the BayesCG algorithm is consistent for performance metrics that are concave functions of the confusion matrix; to the best of our knowledge, this is the first efficient learning algorithm (whose running time is polynomial in the number of classes) that is provably consistent for a large family of multi-class non-decomposable metrics. Also, unlike the brute-force plug-in method, the BayesCG algorithm requires no assumptions on the form of the optimal classifier for the given performance metric and hence on the underlying distribution.

Our consistency result makes use of a novel proof technique based on the convergence analysis of the CG method \cite{Jaggi13}. More specifically, we show that the linear optimization step of the above CG method is solved approximately in the BayesCG algorithm and thus establish a regret bound for the algorithm for smooth concave performance metrics. For performance metrics that are non-smooth concave functions of the confusion matrix, we prescribe applying the BayesCG algorithm to a suitable smooth approximation of these performance metrics; we instantiate and show consistency of this approach for concave performance metrics such as the G-mean, H-mean and Q-mean.

\subsection{Related Work}
\label{subsec:related-work}
There have been several algorithms designed to optimize non-decomposable classification metrics, particularly in the binary classification setting; these include the binary plug-in algorithm that applies an empirical threshold to a class probability estimate \cite{Yang01, Ye+12, Narasimhan+14, NatarajanRavi14, Kotlowski14}, cost-sensitive risk minimization based approaches \cite{Musicant+03, NatarajanRavi14, Parambath+14}, methods that optimize convex and non-convex approximations to the given performance metric \cite{Gao+03, Jansche05, Joachims05, Liu+09, Chinta+13}, and decision-theoretic methods that learn a class probability estimate and compute predictions that maximize the expected value of the performance metric on a test set \cite{Lewis95, Chai05, Musicant+03}. Of these, the plug-in method is known to be consistent for any binary performance metric for which the optimal classifier is threshold-based \cite{Narasimhan+14}, while the cost-sensitive approach is shown to be consistent for the class of fractional-linear performance metrics \cite{NatarajanRavi14}. There have also been results characterizing the optimal classifier for several binary non-decomposable metrics \cite{Ye+12, Lipton+14, Narasimhan+14, NatarajanRavi14}, with the specific form of the classifier available in closed-form for fraction-linear metrics (i.e., metrics that are ratios of linear functions) \cite{NatarajanRavi14}. 

We would also like to point out that there has been some work on designing algorithms for optimizing the F-measure in multi-label classification settings \cite{Petterson+10, Petterson+11, Dembczynski+11, Parambath+14} and consistency results for these methods \cite{Dembczynski+11, Dembczynski+13}, but these results do not apply to the setting considered in this paper. In particular, while the multi-class performance metrics that we seek to optimize are non-decomposable/non-additive over data points, the standard performance metrics of interest in a multi-label setting can indeed be expressed as a sum of losses on individual examples, with each loss on an example potentially being a non-decomposable function of the labels on the example.

%
%
%

\vspace{6pt}
\noindent \textbf{Organization.} We start with some preliminaries and background on non-decomposable performance metrics in Section \ref{sec:prelims}. In Section \ref{sec:opt-classifier}, we give a general framework for analysing multi-class non-decomposable performance metrics and use this framework to derive the form of the optimal classifier for a non-decomposable performance metric. Based on this characterization, we consider a brute-force plug-in method for a multi-class non-decomposable metric in Section \ref{sec:plug-in-brute-force}, and show that this method is consistent. In Section \ref{sec:frank-wolfe}, we design an alternate efficient learning algorithm based on the conditional gradient optimization method, which we show is consistent for a large family of concave non-decomposable metrics. 
All proofs not in the main text are provided in the Appendix.

\section{Preliminaries and Background}
\label{sec:prelims}
\textbf{Notations.}
For any $n\in\Z_+$, we shall denote $[n] = \{1,\ldots,n\}$. For a predicate $\phi$, we shall denote by $\1(\phi)$ the indicator function that takes value 1 if $\phi$ is true and 0 otherwise. The probability simplex of dimension $n$ will be denoted by $\Delta_n=\{\p\in\R_+^n\,|\, \sum_{i=1}^n p_i=1\}$. For a matrix $\G\in\R^{n\times n}$, we will use $\g_y$ to denote the $y^{th}$ column of the matrix, and shall refer to $\|\G\|_1 = \sum_{i=1}^n\sum_{j=1}^n |G_{ij}|$ as the $\ell_1$ norm of $\G$ and to $\|\G\|_\infty = \max_{1\leq i < j \leq n} |G_{ij}|$ as the $\ell_\infty$ norm of $\G$; for any two matrices $\A,\B\in\R^{n\times n}$, we shall denote their component-wise inner product as $\langle A,B \rangle =\sum_{i=1}^n \sum_{j=1}^n A_{i,j}B_{i,j}$. For any set $\CC$, we denote its closure under an appropriate metric space by $\overline{\CC}$. For maximization over integral sets, the notation $\overline{\argmax}$ shall refer to ties being broken in favor of the larger number.


\vspace{6pt}

\noindent\textbf{Problem Setup.}
Let $\X$ be an instance space and $\Y = [n]$ be a set of class labels. We are given a training sample $S = ((x_1,y_1),\ldots,(x_m,y_m)) \,\in\, (\X \times [n])^m$ drawn i.i.d.\ according to an underlying (unknown) probability distribution $D$ over $\X\times [n]$, and the goal in a multi-class classification problem is to learn from these examples a prediction model $\widehat{h}_S:\X\>[n]$, which when given a new instance $x\in\X$, makes a prediction $\hat{y} = \widehat{h}_S(x) \,\in\, [n]$. We will be interested in the more general problem of learning from $S$, a randomized classifier $\widehat{\h}_S:\X\>\Delta_n$ that for each instance outputs a probability distribution over the labels in $[n]$; note that any deterministic classifier can be seen as randomized classifier whose output is always a vertex of the probability simplex $\Delta_n$. In particular, we will consider settings where the performance of $\widehat{\h}_S$ is evaluated using a non-decomposable performance metric $\Pf_D: \Delta_n^\X \> \R_+$ that cannot be expressed as a sum or expectation of losses on individual examples. We shall denote the marginal of $D$ over $\X$ as $D_\X$, the conditional class probabilities for an instance $x$ as $\eta_y(x) = \P(Y=y\,|\,X), \,\forall y \in [n]$, and the prior class probabilities as $\pi_y = \P(Y=y), \,\forall y \in [n]$; for a sample $S$, we shall use $D_S$ to denote the empirical distribution which has its mass uniformly on the instances in $S$.  
~\\
\vspace{-6pt}

\begin{table*}[t]
	\caption{Examples of performance metrics that are (continuous, bounded) functions of the confusion matrix. For any classifier $\h: \X \> \Delta_n$, we denote here $\TPR_y[\h] = \P(h(X) = y \,|\, Y = y)$, $\Prec_y[\h] = \P(Y = y \,|\, h(X) = y)$, and $\pi_y = \P(Y = y)$. Each performance metric here can be expressed as $\Pf_D[\h] = \psi(\text{conf}(\h, D) \equiv \C)$, where the form of $\psi: [0, 1]^{n\times n} \> \R_+$ for a performance metric is given in the fourth column; the last column provides important properties of $\psi$, all of which hold over the set of feasible confusion matrices $\CC_D$ (see \Eqn{eqn:feasible-confusion-matrices}). Note that for any $\C \in \CC_D$, $\pi_y = \sum_{\widehat{y}=1}^n C_{y,\widehat{y}}$.}
	\label{tab:perf-metrics}
	\vspace{3pt}
	\centering
	\small
{\scriptsize
\begin{tabular}{@{}lllll@{}}
		\hline
\vspace{-7pt}\\
		\textbf{metric} & \textbf{Definition} & \textbf{Ref.} &  \multicolumn{1}{c}{$\psi(\C)$} & \multicolumn{1}{c}{\textbf{Properties}} \\
		\hline
		\vspace{-7pt}
		\\
		Accuracy 	   & 	$\sum_{y=1}^n \pi_y\TPR_y$ &
				   &  $\sum_{y=1}^n C_{y,y}$ & Linear 
\\	[7pt]
		AM ($1-\text{BER}$) & $\frac{1}{n}\sum_{y=1}^n\TPR_y $ &
\cite{Cheng+02}
 &   $\frac{1}{n}\sum_{y=1}^n
\frac{C_{y,y}}{\sum_{\widehat{y} = 1}^n C_{y,\widehat{y}}}$	& Linear  
\\	[7pt]
		Binary $\text{F}_1$-metric &
$\frac{2}{\frac{1}{\Prec_1} + \frac{1}{\TPR_1}}$ & \cite{Manning08} &   $\frac{2C_{2,2}}{2C_{2,2} + C_{1,2} + C_{2,1}}$	& Non-concave, Pseudo-linear 
\\[7pt]		
		Jaccard Coefficient (JAC) & $\frac{\pi_2\TPR_2}{\pi_2 \,+\, \pi_1(1-\TPR_1)}$ & \cite{SokolovaLapalme09} & $\frac{C_{2,2}}{C_{2,2} + C_{2,1} + C_{1,2 }}$ & 
  Non-concave, Pseudo-linear  
\\[7pt]		
		AMS metric & {-} & \cite{Cowan+11} & {\tiny$\sqrt{2\Big((C_{12}+C_{22})\log\Big(1+\frac{C_{22}}{C_{12}}\Big)-C_{22}\Big)}$} 
 & Convex 
\\	[7pt]
		Micro $\text{F}_1$-metric & 
- & \cite{Kim+13} &   $\frac{2\sum_{y = 2}^n C_{yy}}{2\sum_{y = 2}^n C_{yy} \,+\, \sum_{y = 1}^n\sum_{\widehat{y} \ne y} C_{y,\widehat{y}}}$	& Non-concave, Pseudo-linear  
\\	[7pt]
		Macro $\text{F}_1$-metric &
$\frac{1}{n}\sum_{y=1}^n\frac{2}{\frac{1}{\Prec_y} + \frac{1}{\TPR_y}}$ & \cite{GopalYang13}  &   $\frac{1}{n}\sum_{y=1}^n\frac{2C_{y,y}}{\sum_{\widehat{y} = 1}^n C_{y,\widehat{y}} \,+\, \sum_{\hy=1}^nC_{\hy,y}}$	& Non-concave
\\[7pt]
		H-Mean (HM) & $\frac{n}{\sum_{y=1}^n \frac{1}{\TPR_y}}$ &
\cite{Kennedy:2009} &  $n\Big({\sum_{y=1}^n \frac{\sum_{\widehat{y} = 1}^n C_{y,\widehat{y}}}{C_{y,y}}}\Big)^{-1}$ & Concave, Non-smooth 
\\[7pt]
		Q-Mean (QM) & $1-\sqrt{\frac{1}{n}\sum_{y=1}^n(1-\TPR_y)^2}$ &
\cite{Lawrence:98} & 
$1-\sqrt{\frac{1}{n}\sum_{y=1}^n\Big(1-\frac{C_{y,y}}{\sum_{\widehat{y} = 1}^n C_{y,\widehat{y}}}\Big)^2}$ & Concave, Non-smooth 
\\	[7pt]
		G-Mean (GM) & $\big(\prod_{y=1}^n \TPR_y\big)^{1/n}$ &
\cite{KubatMa97,Daskalaki+06} &  $\Big(\prod_{y=1}^n
\frac{C_{y,y}}{\sum_{\widehat{y} = 1}^n C_{y,\widehat{y}}}\Big)^{1/n}$ & Concave, Non-smooth  
\\[7pt]		
		Min-Max metric & $\min_{y \in [n]}\TPR_y$ & \cite{PoorVincent94} & 
$\min_{y \in [n]}\frac{C_{y,y}}{\sum_{\widehat{y} = 1}^n C_{y,\widehat{y}}}$ & Concave, Non-differentiable 
\\[7pt]
		\hline
	\end{tabular} 
}	\vspace{-17pt}
\end{table*}

\noindent \textbf{Multi-class Non-decomposable Performance Metrics.} Let us first define for a deterministic classifier $h:\X\>[n]$ and distribution $D$, the confusion matrix $\conf(h,D) \in [0, 1]^{n\times n}$ as  
\[
\big[\conf(h,D)\big]_{i,j} = \E_{(X,Y)\sim D}\big[\1(Y=i, \,h(X)=j)\big], \hspace{1em} \forall i,j\,\in\,[n];
\vspace{-3pt}
\]
the corresponding confusion matrix for a randomized classifier $\h:\X\>[n]$ is given by 
\[
\big[\conf(\h,D)\big]_{i,j} = \E_{(X,Y)\sim D} \big[h_j(X)\cdot \1(Y=i) \big], \hspace{1em} \forall i,j\,\in\,[n].
\vspace{-3pt}
\]
In this paper, we shall be interested in non-decomposable performance metrics that can expressed as a continuous and bounded function $\psi:[0,1]^{n\times n}\>\R_+$ of the confusion matrix:
\begin{equation}
\label{eqn:general_eval}
\cP_D[\h] \,=\, \psi( \conf(\h,D)).
\vspace{-3pt}
\end{equation}
For example, the macro $\text{F}_1$-measure used widely in text retrieval can be expressed as a function $\psi_{\text{F}_1}(\C) \,=\, \frac{1}{n}\sum_{y=1}^n\frac{C_{y,y}}{\sum_{\widehat{y} = 1}^n C_{y,\widehat{y}} \,+\, \sum_{\hy=1}^nC_{\hy,y}}$ of the confusion matrix $\C \in [0, 1]^{n\times n}$. Table \ref{tab:perf-metrics} contains several examples of performance metrics that are functions of the confusion matrix.\footnote{For all performance metrics considered in this paper, higher values indicate better performance.}$^{,}$\footnote{In the setting considered here, the goal is to maximize a performance metric that can be expressed as a (non-decomposable) function of expectations; this is referred to by Ye et al. (2012) \cite{Ye+12} as the expected utility maximization setup and is different from the decision-theoretic setting that they consider, where one looks at the expectation of a non-decomposable performance metric on $m$ examples, and seeks to maximize its limiting value as $m \> \infty$.} Throughout this paper, we shall use the term performance metric to refer to both $\Pf$ and $\psi$.
~\\
\vspace{-6pt}

\noindent \textbf{$\psi$-consistency.} We now consider the optimal value of performance metric $\cP$ over all randomized classifiers:
\[
\cP^*_D \,=\, \sup_{\h : \X\>\Delta_n} \cP_D[\h],
\]
and shall refer to the classifier attaining the above value, if one exists, as the $\psi$-optimal classifier. One can then define the $\psi$-regret of classifier $\h$ as
\[
\reg^\psi_D[\h] \,=\, \cP_D[\h] \,-\, \cP^*_D.
\] 
A learning algorithm that takes a training sample $S$ drawn i.i.d. from $D^m$ and outputs a classifier $\widehat{\h}_S$ is said to be $\psi$-consistent if the $\psi$-regret of classifier $\widehat{\h}_S$ goes to zero in probability:
\[
\reg^\psi_D[\widehat{\h}_S] \,\xrightarrow{P}\, 0,
\]
where the convergence in probability is over the random draw of $S$ from $D^m$.\footnote{We say $\phi(S)$ converges in probability to $a \in \R$, written as $\phi(S) \xrightarrow{P} a$, if ${\forall \epsilon > 0}$, ${\P_{S \sim D^m}(|\phi(S) - a| \geq \epsilon) \rightarrow 0} \text{ as } m \rightarrow \infty$.}

\vspace{-6pt}
~\\
\noindent\textbf{Optimal Classifier for Decomposable Metrics.} 
While in general, it is not clear if there exists a classifier that attains the optimal value of a given performance metric $\Pf_D[\h] = \psi(\text{conf}[\h])$, it is well-known that when $\psi$ is a linear function (i.e., $\Pf_D$ can be expressed as an expectation of a loss on individual example), a $\psi$-optimal classifier always exists. In particular, if $\psi$ takes the form 
\[
\psi_\G(\C) = \sum_{y=1}^n \sum_{\hat y=1}^n G_{y,\hat y} C_{y,\hat y} =  \langle \C, \G \rangle,
\]
for some matrix $\G \in \R_+^{n \times n}$, then any classifier $\h^*:\X \> \Delta_n$ that satisfies the following condition is $\psi_\G$-optimal:
\begin{equation}
\label{eqn:Bayes-decomposable}
h^{*}_i(x) > 0  ~~\text{ only if }~~ i \in \argmax_{y \in [n]} \g_{y}^\top {\boldeta}(x)  ~.
\end{equation}
It is seen that there always exists a deterministic classifier that satisfies the above condition. Also, it is worth noting that maximizing the above performance metric is equivalent to solving a cost-sensitive classification problem, with the costs given by the the negative of the `gain' matrix $\G$. 

\begin{figure*}[t]
\begin{algorithm}[H]
\caption{
Plug-in Algorithm for Binary Non-decomposable Performance Metric.
}
\label{algo:plug-in-binary}
\begin{algorithmic}[1]
\STATE \textbf{Input:} $S=((x_1,y_1), \ldots, (x_m, y_m)) \in (\X\times [2])^m$,  $\psi: [0, 1]^{2\times 2} \> \R_+$
\STATE Split $S$ into two sets $S'$ and $S''$ with sizes $m_1 = \lfloor(1-\alpha)m\rfloor$ and $m_2 = \lceil\alpha m\rceil$. 
\STATE Learn $\hat{\eta}_{S'}=\CPE(S')$, where $\CPE: \cup_{m=1}^\infty (\X \times [2])^m \> [0,1]^\X$ is a suitable CPE algorithm 
\STATE $\widehat{t}_S \,\in\, \argmax_{t \in [0, 1]} \, {\Pf}_{D_{S''}}[\widehat{\h}_{t}]$, where $\widehat{\h}_{t}(x) = 
\begin{cases}
[1, 0]^\top & \text{if }\widehat{\eta}_{S'}(x) \leq  t\\
[0, 1]^\top & \text{otherwise}
\end{cases}$
\\[2pt]
\STATE \textbf{Output:} $\widehat{\h}_{\widehat{t}_S}$ 
\end{algorithmic}
\end{algorithm}	
\vspace{-0.7cm}
\end{figure*}

\vspace{-6pt}
~\\
\noindent\textbf{Plug-in Algorithm for Decomposable Metrics.} A standard approach for maximizing a decomposable metric (or equivalently solving a cost-sensitive classification problem) is the plug-in method, where one first obtains a class probability estimation (CPE) model $\widehat{\boldeta}_S: \X \> \Delta_n$ from the given training sample $S$ and constructs a classifier $\widehat{\h}_S(x) =  \overline{\argmax}_{y \in [n]}\, \g_{y}^\top \widehat{\boldeta}_{S}(x)$ for any instance $x$. This approach can be shown to be $\psi_\G$-consistent if the CPE algorithm used to learn $\widehat{\boldeta}_S$ is such that $\E_X \big[ \big\|\widehat{\boldeta}_{S'}(X) - \boldeta(X)\big\|_1\big]  \xrightarrow{P}  0$ \cite{Clemencon+11} (which is indeed the case for any algorithm that performs a regularized empirical risk minimization of a proper loss such as the logistic loss \cite{Agarwal13,Menon+13}).

\vspace{-6pt}
~\\
\noindent\textbf{Known Results for Binary Non-decomposable Performance Metrics.} We now summarize what is understood about the the optimal classifier for binary non-decomposable performance metrics and about the consistency properties of learning algorithms for these metrics. It is known that, under a continuous distribution, the optimal classifier for a binary monotonic non-decomposable metric is obtained by placing a suitable threshold on the posterior class probability function \cite{Narasimhan+14}. For certain specific performance metrics, such as those that are fractional-linear/ratio of linear functions (e.g., binary F-measure and JAC measure) \cite{Ye+12, Zhao+14, Lipton+14, NatarajanRavi14}, the geometric mean of precision and recall \cite{Narasimhan+14}, and the approximate median sign (AMS) metric \cite{Kotlowski14}, this characterization holds even without the continuity assumption on the distribution; for some of these metrics, the exact form of the threshold is also available in closed-form \cite{NatarajanRavi14, Kotlowski14}. It is also known that a plug-in algorithm that constructs a classifier by assigning an empirical threshold to a suitable class probability estimate (see Algorithm \ref{algo:plug-in-binary}) is statistically consistent with respect to any binary non-decomposable metric for which the optimal classifier is of the above thresholded form \cite{Narasimhan+14,NatarajanRavi14}; a similar result has also been shown for a cost-sensitive risk minimization based approach for fractional-linear metrics \cite{NatarajanRavi14}.

While there has been a lot of work on binary non-decomposable metrics as seen above, little is known about how these results extend to the multi-class case. In particular, what is the form of the optimal classifier for a general multi-class non-decomposable metric? How does the plug-in and cost-sensitive risk minimization based algorithms for binary performance metrics, which essentially need to tune a single parameter, generalize to the multi-class case, where the number of parameters needed to be tuned grows with the number of classes? In this paper, we address these questions.

\vspace{-6pt}
~\\
Before we proceed further, we will find it convenient to define for any given function $\bmu: \X \> \R^n$, the set of weighted argmax classifiers obtained by a gain matrix $\G \in \R_+^{n \times n}$ on $\mu$:
\[
{\H}_\bmu = \big\{\h:\X\>\Delta_n \,\,|\,\, \exists\, \G\in\R^{n\times n} \text{ s.t. } \forall x \in \X, \, h_i(x) = 1 \,\text{ if }\, i = \overline{\argmax}_{y\in[n]} [\G \bmu(x)]_y \big\}.
\]
\noindent Finally, a function $f: \R^{d \times d} \> \R$ is said to be $L$-Lipschitz w.r.t. the $\ell_1$ norm over $\mathcal{M} \subseteq \R^{d\times d}$, for some $L> 0$, if
\[
|f(\M_1) - f(\M_2)| \,\leq\, L\|\M_1 - \M_2\|_1, \,\,\forall\, \M_1, \,\M_2 \,\in\,\mathcal{M},
\]
and is $\beta$-smooth w.r.t. the $\ell_1$ norm over $\mathcal{M} \subseteq \R^{d\times d}$, for some $\beta > 0$, if
\[
\|\nabla f(\M_1) - \nabla f(\M_2)\|_\infty \,\leq\, \beta\|\M_1 - \M_2\|_1, \,\,\forall\, \M_1, \,\M_2 \,\in\, \mathcal{M}.
\]

\section{Characterization of the Optimal Classifier for a General Multi-class Performance Metric}
\label{sec:opt-classifier}

We start by providing a generic framework for studying a multi-class non-decomposable performance metric, where we view the problem of finding the optimal classifier for a non-decomposable metric as an optimization problem over the space of all confusion matrices that are attainable under the given distribution. Using this framework, we give a characterization of the optimal classifier for a non-decomposable metric; in particular, we show that under a continuous distribution, the optimal classifier for any multi-class non-decomposable performance metric (that satisfies a mild condition) can be obtained by maximizing a decomposable performance metric, whose gain matrix is given by the gradient of non-decomposable metric at the optimal confusion matrix. To our knowledge, this is the first such result for a general multi-class non-decomposable metric, generalizing a previous result for binary non-decomposable metrics \cite{Narasimhan+14} and in addition also recovering previous results on the form of the optimal classifier for several performance metrics \cite{Lipton+14,NatarajanRavi14,Kotlowski14}. 
~\\[-6pt]

\noindent \textbf{Feasible confusion matrices.} We begin by defining the set of feasible confusion matrices for a distribution $D$ as the set of all confusion matrices achievable by a randomized classifier under $D$:
\begin{eqnarray}
 \CC_D &=& \{\C \in [0,1]^{n\times n} :\, \C=\conf(\h,D) \text{ for some }\h:\X\>\Delta_n\}.
 \label{eqn:feasible-confusion-matrices}
\end{eqnarray}
Note that every matrix $\C\in\CC_D$ is such that its row sums are equal to the prior probabilities, i.e. $\sum_{\hy=1}^n C_{y, \hy} = \pi_y, \,\forall y \in [n]$. It can be shown that this set is convex.
\begin{prop}[\textbf{Convexity of $\CC_D$}]
\label{prop:convexity-cc}
$\CC_D$ is a convex set.
\end{prop}
\noindent The problem of finding the optimal classifier for the given performance metric can now be cast as an optimization problem over $\CC_D$; we shall shortly see that this viewpoint is useful in both characterizing the optimal classifier for the performance metric and in designing consistent learning algorithms for the metric.

\vspace{-6pt}
~\\
\noindent We next make the following continuity assumption on $D$, which is essentially a multi-class extension of a similar assumption on $D$ in \cite{Narasimhan+14} (in the binary label setting).

\vspace{-6pt}~\\
\noindent\textbf{Assumption A (Continuity of $D$).} Let $U$ be a random variable distributed uniformly over the simplex $\Delta_n$, and let $\mu$ be a base measure over $\Delta_n$ such that $\mu(\cA)=\P(U\in\cA), \,\forall \cA\subseteq\Delta_n$. Let $\nu$ denote the probability measure that is associated with the random variable ${\boldeta}(X)$. We will say that a distribution $D$ satisfies Assumption A if $\nu$ is absolutely continuous w.r.t. $\mu$.

\vspace{-6pt}~\\
\noindent We shall also make a mild assumption on $\psi$ that is satisfied by all performance metrics in \Tab{tab:perf-metrics} except the min-max metric.
~\\\vspace{-10pt}~\\
\noindent\textbf{Assumption B.} We will say that $\psi: [0, 1]^{n\times n} \> \R_+$ satisfies Assumption B w.r.t. distribution $D$ if it is continuous, differentiable and bounded over $\CC_D$, and is strictly increasing in the diagonal elements of its argument and non-increasing in the non-diagonal elements of its argument. 

\vspace{-6pt}~\\
Under the above assumptions on $D$ and $\psi$, we now show that a $\psi$-optimal classifier always exists and can be obtained by maximizing a decomposable performance metric constructed from the gradient of $\psi$ at the optimal confusion matrix.

\begin{thm}[\textbf{Characterization of $\psi$-optimal Classifier for a General Multi-class Non-decomposable Metric Under Continuous Distributions}]
\label{thm:suff-cond-equivalence}
Let distribution $D$ satisfy Assumption A, and $\psi: [0, 1]^{n\times n} \> \R_+$ satisfy Assumption B w.r.t. $D$. Then there exists a classifier $\h^*: \X \> \Delta_n$ that is $\psi$-optimal. Furthermore, for  $\G^*=\nabla\psi(\conf(\h^*,D))$, 
we have
\[
\emptyset \,\ne\, \underset{\h: \X \> \Delta_n}{\argmax} \,\langle \G^*, \conf(\h,D) \rangle \,\,\subseteq\,\, \underset{\h: \X \> \Delta_n}{\argmax} \,\psi(\conf(\h, D)),
\]
and thus any classifier $\tilde{h}: \X \> \Delta_n$ of the following form is $\psi$-optimal:
\begin{equation*}
\tilde{h}_i(x) > 0  ~~\text{ only if }~~ i \in \underset{y \in [n]}{\argmax}\,\,\g_{y}^{*\top} {\boldeta}(x).
\end{equation*}
\end{thm}

The above theorem is a multi-class generalization of the result in \cite{Narasimhan+14} for binary monotonic performance metrics, and in addition also gives the precise form of the optimal classifier for the given performance metric. By a simple application of this theorem, we recover previous results on the form of the optimal classifier for performance metrics that fractional-linear \cite{NatarajanRavi14} such as the F-measure and Jaccard coefficient \cite{Lipton+14}, and also for the AMS metric \cite{Kotlowski14}. 

Before we prove \Thm{thm:suff-cond-equivalence}, we will find it useful to state the following lemma.
\begin{lem}[\textbf{Uniqueness of Optimal Confusion Matrix for Gain Matrices Obtained from Gradients of $\psi$}]
\label{lem:opt-in-closure-same-as-opt}
Under the assumptions on $D$ and $\psi$ in \Thm{thm:suff-cond-equivalence}, for any $\C^* \in \CC_D$, we have
\[
\argmax_{\C\in\overline{\CC}_D}\, \langle \nabla\psi(\C^*), \C \rangle \,=\, \argmax_{\C\in\CC_D} \langle \nabla\psi(\C^*), \C \rangle.
\]
Moreover, the above set is a singleton.
\end{lem}
\noindent The proof of \Thm{thm:suff-cond-equivalence} then follows from the first order necessary conditions for optimality of a confusion matrix and the above result.

\begin{proof}[Proof of \Thm{thm:suff-cond-equivalence}]

We shall first show that there exists a $\psi$-optimal classifier. By compactness of $\overline{\CC}_D$, we know that there exists $\C^* \in \overline{\CC}_D$ such that
$$\psi(\C^*) = \max_{\C\in \overline{\CC}_D} \psi(\C) =  \sup_{\C\in \CC_D} \psi(\C).$$
It remains to be shown that there exists a classifier that achieves this confusion matrix, i.e., $\C^* \in \CC_D$. For this, we note from the first order necessary condition for optimality of $\C^*$, given convexity of $\overline{\CC}_D$ (see \Prop{prop:convexity-cc}), that
\begin{eqnarray}
\langle \nabla\psi(\C^*), \C^* \rangle \geq \langle \nabla\psi(\C^*),\C \rangle, \hspace{1em} \forall \C\in\overline{\CC}_D.
\label{eqn:suff-cond-1}
\end{eqnarray}
The above equation along with \Lem{lem:opt-in-closure-same-as-opt} implies that
\[
\argmax_{\C\in\overline{\CC}_D}\, \langle \nabla\psi(\C^*), \C \rangle \,=\, \argmax_{\C\in\CC_D} \langle \nabla\psi(\C^*), \C \rangle  \,=\, \{\C^*\}\;.
\]
Thus $\C^*\in\CC_D$ and hence there exists a clasifier $\h^*:\X\>\Delta_n$ such that $\C^*=\conf(\h^*,D)$. This completes the proof of existence of a $\psi$-optimal classifier. 

Next for  $\G^* = \nabla\psi(\C^*)$, we further have
\[
\underset{\h: \X \> \Delta_n}{\argmax} \,\langle \G^*, \conf(\h,D) \rangle \,\, 
= \{\h:\X\>\Delta_n:\conf(\h,D)=\C^*\} \,\, 
\subseteq\,\, \underset{\h: \X \> \Delta_n}{\argmax} \,\psi(\conf(\h, D)).
\]
Clearly, a classifier $\tilde{\h}: \X \> \Delta_n$ that maximizes the linear performance metric $\langle \G^*, \conf(\cdot,D) \rangle$ is also $\psi$-optimal; as seen in \Eqn{eqn:Bayes-decomposable}, such a classifier takes the form given in the theorem statement.
\end{proof}

\begin{rem}[\textbf{Necessity of continuity Assumption A on $D$}]
\label{rem:necessity-ass-A}
We note here that for the above characterization to hold for a general non-decomposable performance metric, the continuity assumption on distribution $D$ (Assumptions A) is indeed necessary. We illustrate this fact for the H-mean performance metric by constructing a simple distribution that does not satisfy this assumption, and where a classifier of the form in the theorem statement is not necessarily optimal. Consider the following distribution $D$ over $\{x\} \times \{1,2\}$ with $\boldeta(x) = \big[\frac{1}{2},\, \frac{1}{2}\big]^\top$. It can be seen that the unique optimal classifier for the H-mean performance metric is $\h^*(x) = \big[\frac{1}{2},\, \frac{1}{2}\big]^\top$, whose confusion matrix $\C^*$ and the gradient of $\psi$ at $\C^*$ are given by:
$$\C^* \,=\,
\begin{bmatrix}
\frac{1}{4} & \frac{1}{4} \\[3pt]
\frac{1}{4} & \frac{1}{4} 
\end{bmatrix};
~~
\G^* \,=\, \nabla \psi(\C^*) \,=\,
\begin{bmatrix}
~\,\,\frac{1}{2} & -\frac{1}{2} \\[3pt]
-\frac{1}{2} & ~\,\,\frac{1}{2} 
\end{bmatrix}.
$$
Clearly, any classifier $\h: \X \> \Delta_2$ will have $\langle \G^*, \conf(\h, D)\rangle = 0$; hence 
\[
\{\h: \X \> \Delta_2\}\,=\, \underset{\h: \X \> \Delta_2}{\argmax} \,\langle \G^*, \conf(\h,D) \rangle \,\,\supset \,\, \underset{\h: \X \> \Delta_2}{\argmax} \,\psi(\conf(\h, D)) \,=\, \{\h^*\}.
\]
\end{rem}

It is worth noting that for certain restricted families of performance metrics, the characterization in \Thm{thm:suff-cond-equivalence} holds even without Assumption A on the distribution; this is the case, for example, when $\psi$ is fractional-linear (e.g., F-measure, JAC) \cite{NatarajanRavi14, Parambath+14} and is convex (e.g., AMS metric) \cite{Kotlowski14}. 

\begin{rem}[\textbf{Extension to the min-max metric}]
A result similar to the one in \Thm{thm:suff-cond-equivalence} also holds for the min-max metric, where it is well known from classical detection theory (in particular, from min-max hypothesis testing) that the optimal classifier for this metric is obtained by maximizing a decomposable metric with an appropriate gain matrix \cite{PoorVincent94}. In fact, one can show that if $\h^*$ is an optimal classifier for the min-max metric $\psi^{\MM}$, and $\G^*$ is in the sub-differential of $\psi^\MM$ at $\conf(\h^*,D)$, then
\[
\emptyset \,\ne\, \underset{\h: \X \> \Delta_n}{\argmax} \,\langle \G^*, \conf(\h,D) \rangle \,\,\subseteq\,\, \underset{\h: \X \> \Delta_n}{\argmax} \,\psi^\MM(\conf(\h, D)).
\]
\end{rem}



\section{A Consistent Plug-in Method for Multi-class Non-decomposable Metrics Based on a Brute-force Search}
\label{sec:plug-in-brute-force}
Based on the above characterization of the optimal classifier of a non-decomposable metric, we now consider a simple plug-in based learning algorithm for a multi-class non-decomposable metric that uses a brute-force search over gain matrices; this approach can be seen as a natural extension of the binary plug-in method in Algorithm \ref{algo:plug-in-binary}. We show that this method is consistent with respect to a general non-decomposable metric, and also provide an explicit regret bound for this method for the special case of performance metrics that exhibit a certain convexity-like property. In the next section, we design an alternate efficient learning algorithm based on the conditional gradient algorithm which is consistent for a large family of non-decomposable metrics.

Clearly, if the optimal confusion matrix $\C^*$ for a multi-class non-decomposable metric $\psi$ is known apriori, one can learn a simple plug-in classifier by applying the gradient of $\psi$ at $\C^*$ to a suitable class probability estimator. In the absence of knowledge of $\C^*$, a natural first-cut approach would be to perform a brute-force search over all gain matrices with bounded entries\footnote{Since a plug-in classifier constructed from a gain matrix is invariant to scaling of entries of the matrix, it suffices to perform the search over gain matrices with bounded entries.}, and pick the one for which the resulting plug-in classifier yields maximum performance value on a held-out part of the training set (see Algorithm \ref{algo:brute-force-plug-in-non-decomposable}). While for the binary case ($n = 2$), this brute-force search essentially reduces to a search over thresholds (on the class probability estimate) that can be performed efficiently in time linear in the number of held-out instances (as seen in Algorithm \ref{algo:plug-in-binary}), for the general multi-class case, it is not clear if an exact search is tractable; in practice, this maximization over gain matrices can be performed approximately by considering only a finite number of matrices obtained from a fine-grained grid.

We now show that (under a continuous distribution) the brute-force plug-in method is statistically consistent with respect to the given performance metric. 

\begin{figure*}[t]
{
\begin{algorithm}[H]
\caption{
Brute-force Plug-in Algorithm for Multi-class Non-decomposable Performance Metric
}
\label{algo:brute-force-plug-in-non-decomposable}
\begin{algorithmic}
\STATE \textbf{Input:} $S=((x_1,y_1), \ldots, (x_m, y_m)) \in (\X\times [n])^m$, $\psi: [0,1]^{n \times n} \> \R_+$
\STATE \textbf{Parameter:} $\alpha\in(0,1)$
\STATE Split $S$ into two sets $S'$ and $S''$ with sizes $m_1 = \lfloor (1-\alpha) m \rfloor$ and $m_2 = \lceil \alpha m \rceil$. 
\STATE Learn $\hat{\boldeta}_{S'}=\CPE(S')$, where $\CPE: \cup_{m=1}^\infty (\X \times [n])^m \> \Delta_n^\X$ is a suitable CPE algorithm 
\STATE $\forall \G \,\in\, [-1, 1]^{n \times n},$ define $\widehat{\h}_{\G}:\X \> \Delta_n$ such that $\big[\widehat{\h}_{\G}(x)\big]_i = 1 \,\text{ if  }\, i = \overline{\argmax}_{y \in [n]} {\g}_{y}^\top \widehat{\boldeta}_{S'}(x)$
\STATE $\widehat{\G}_{S} \,\in\, \underset{\G \,\in\, [-1, 1]^{n \times n}}{\argmax}\,\, \cP_{D_{S''}}\big[\widehat{\h}_{\G}\big]$
\STATE \textbf{Output:} $\widehat{\h}_{S} \,\equiv\, \widehat{\h}_{\widehat{\G}_{S}}$
\end{algorithmic}
\end{algorithm}
\vspace{-0.7cm}
}
\end{figure*}

\begin{thm}[\textbf{Consistency of Brute-force Plug-in Algorithm for Multi-class Non-decomposable Metrics}]
\label{thm:consistency-brute-force}
Let $D$ satisfy Assumption A, and $\psi: [0, 1]^{n\times n} \> \R_+$ satisfy Assumption B w.r.t. $D$. If $\widehat{\h}_S$ is the classifier learned by Algorithm \ref{algo:brute-force-plug-in-non-decomposable} using training sample $S = (S', S'') \in (\X \times [n])^m$ with parameter $\alpha \in (0,1)$, and the $\CPE$ algorithm used in Algorithm \ref{algo:brute-force-plug-in-non-decomposable} is such that $\E_X \big[ \big\|\widehat{\boldeta}_{S'}(X) - \boldeta(X)\big\|_1\big]  \xrightarrow{P} 0$, then $\reg_D^\psi[\widehat{\h}_{S}] \,\xrightarrow{P} \, 0$ (as $m \rightarrow \infty$). 
\end{thm}

The above guarantee applies to all performance metrics in \Tab{tab:perf-metrics}. Before we prove this result, we state a couple of lemmas; in the first lemma, we consider a classifier obtained by applying a fixed gain matrix to a class probability estimation model, and show convergence of the entries of the confusion matrix for this classifier to those of a classifier obtained by applying the gain matrix to the true class probability function; in the second lemma, we give a uniform convergence bound for the confusion matrix of a set of weighted argmax classifiers.

\begin{lem}[\textbf{Convergence of $\conf$ for fixed gain matrix}]
\label{lem:fixed-gain-matrix-conf}
Let $D$ satisfy Assumption A. Let $\widehat{\boldeta}_{\tilde{S}}: \X \> \Delta_n$ be a class probability estimation model learned using a sample $\tilde{S}$ drawn i.i.d. from $D^{\tilde{m}}$. For a fixed gain matrix $\G \in [0,1]^{n \times n}$ such that no two columns are identical, let $\h_{\G}: \X \> \Delta_n$ and $\widehat{\h}_{\G}: \X \> \Delta_n$ be classifiers constructed as follows: $\big[{\h}_{\G}(x)\big]_i = 1 \text{ if  } i = \overline{\argmax}_{y \in [n]} {\g}_{y}^\top {\boldeta}(x), \,\forall x \in \X$ and $\big[\widehat{\h}_{\G}(x)\big]_i = 1 \text{ if  } i = \overline{\argmax}_{y \in [n]} {\g}_{y}^\top \widehat{\boldeta}_{\tilde{S}}(x), \,\forall x \in \X$. If $\widehat{\boldeta}_{\tilde{S}}$ is such that $\E_X \big[ \big\|\widehat{\boldeta}_{\tilde{S}}(X) - \boldeta(X)\big\|_1\big] \xrightarrow{P} 0$, then $\forall i,j,\,\big[\conf(\widehat{\h}_{\G}, D)\big]_{ij} \,\xrightarrow{P} \, \big[\conf({\h}_{\G}, D)\big]_{ij}$ (as $m \rightarrow \infty$).
\end{lem}

\begin{lem}[\textbf{Uniform Convergence Generalization Bound for $\conf$ Over $\H_{{\bmu}}$}]
\label{lem:uconvg-conf}
Let ${\bmu}: \X \> \R^n$ be a fixed function and ${\tilde{S}} \in (\X \times [n])^{\tilde{m}}$ be a sample drawn i.i.d. according to $D^{\tilde{m}}$. For any $\delta \in [0,1]$, we have with probability at least $1-\delta$ (over draw of $\tilde{S}$ from $D^{\tilde{m}}$),
\[
\sup_{\h \,\in\, \H_{{\bmu}}}\,\big\|\conf(\h, D) \,-\, \conf(\h, D_{\tilde{S}})\big\|_\infty \,\,\leq\,\, 
C\sqrt{\frac{n^2\log(n)\log(\tilde m) + \log(n^2/\delta)}{\tilde m}},
\]
where $C > 0$ is a distribution-independent constant.
\end{lem}

\noindent We are now ready to prove Theorem \ref{thm:consistency-brute-force}.

\begin{proof}[Proof of Theorem \ref{thm:consistency-brute-force}]
By Theorem \ref{thm:suff-cond-equivalence}, a $\psi$-optimal classifier exists. Let $\h^*: \X \> \Delta_n$ be one such classifier and let $\G^*=\nabla\psi(\conf(\h^*,D))$. Further, let $\h_{\G^*}: \X \> \Delta_n$ be a classifier such that $\big[\h_{\G^*}(x)]_i = 1$  if $i = \overline{\argmax}_{y \in [n]} \g_{y}^{*\top} {\boldeta}(x)$; then again by Theorem \ref{thm:suff-cond-equivalence}, $\Pf_D[\h_{\G^*}] = \Pf_D[\h^*]$. Also let $\widehat\h_{\G^*}: \X \> \Delta_n$ be  such that $\big[\widehat\h_{\G^*}(x)]_i = 1$  if $i = \overline{\argmax}_{y \in [n]} \g_{y}^{*\top} {\hat\boldeta_{S'}}(x)$. Thus,
\begin{eqnarray*}
{\reg_D^\psi[\widehat{\h}_S]}
&=& \Pf_D[\h^*] \,-\, \Pf_D[\widehat{\h}_S]\\
&=& \Pf_D[\h_{\G^*}] \,-\, \Pf_D[\widehat{\h}_S]\\
&=& \Pf_D[\h_{\G^*}] \,-\, \Pf_D[\widehat{\h}_{\G^*}]  \,\,+\,\, \Pf_D[\widehat{\h}_{\G^*}] \,-\, {\Pf}_{D_{S''}}[\widehat{\h}_{\G^*}] \,\,+\,\, {\Pf}_{D_{S''}}[\widehat{\h}_{\G^*}] \,-\, \Pf_D[\widehat{\h}_S]\\
&\leq& \Pf_D[\h_{\G^*}] \,-\, \Pf_D[\widehat{\h}_{\G^*}]  \,\,+\,\, \Pf_D[\widehat{\h}_{\G^*}] \,-\, {\Pf}_{D_{S''}}[\widehat{\h}_{\G^*}] \,\,+\,\, {\Pf}_{D_{S''}}[\widehat{\h}_{S}] \,-\, \Pf_D[\widehat{\h}_S]\\
&\leq& \Pf_D[\h_{\G^*}] \,-\, \Pf_D[\widehat{\h}_{\G^*}]  \,\,+\,\,
		\sup_{\h \,\in\, \H_{\widehat{\boldeta}_{S'}}}\big(\Pf_D[\h] \,-\, {\Pf}_{D_{S''}}[\h]\big) \,\,+\,\,
		\sup_{\h \,\in\, \H_{\widehat{\boldeta}_{S'}}}\big({\Pf}_{D_{S''}}[\h] \,-\, \Pf_D[\h]\big)\\
&=& \underbrace{\Pf_D[\h_{\G^*}] \,-\, \Pf_D[\widehat{\h}_{\G^*}]}_{\term_A}  \,\,+\,\,
		2\underbrace{\sup_{\h \,\in\, \H_{\widehat{\boldeta}_{S'}}}\,\big|\Pf_D[\h] \,-\, {\Pf}_{D_{S''}}[\h]\big|}_{\term_B},
\end{eqnarray*}
where the fourth step follows by definition of $\widehat{\h}_{S}$. By assumption B on $\psi$, the matrix $\G^*$ has no two identical columns, and hence by Lemma \ref{lem:fixed-gain-matrix-conf} we have that $\conf(\widehat{\h}_{\G^*})$ converges to $\conf(\h_{\G^*})$ as $m$ goes to $\infty$. Along with the continuity of $\psi$, this ensures that  $\term_A \xrightarrow{P} 0$.  By suitably conditioning on $S'$ and using the uniform convergence bound in Lemma \ref{lem:uconvg-conf}, one gets $\term_B \xrightarrow{P} 0$. 
\end{proof}	

For a special class of performance metrics that satisfy a certain convexity-like property, we have an explicit regret bound guarantee for the brute-force plug-in method.
\begin{thm}[\textbf{Regret Bound for Brute-force Plug-in Algorithm for Convex-like Non-decomposable Metrics}]
\label{thm:regret-bound-brute-force-plug-in}
Let $D$ satisfy Assumption A, and $\psi: [0, 1]^{n\times n} \> \R_+$ satisfy Assumption B w.r.t. $D$. Furthermore, let $\psi$ be $L$-Lipschitz w.r.t. the $\ell_1$ norm over $\CC_D$, and be such that there exists $\xi > 0$ such that $\psi(\C) - \psi(\C') \,\leq\, \xi\langle \nabla\psi(\C), \C - \C' \rangle,\,\forall \C,\, \C' \in \CC_D$. If $\widehat{\h}_S$ is the classifier learned by Algorithm \ref{algo:brute-force-plug-in-non-decomposable} using training sample $S = (S', S'') \in (\X \times [n])^m$ with parameter $\alpha \in (0,1)$, then for any $\delta \in [0,1]$, we have with probability at least $1-\delta$ (over draw of $S$ from $D^m$):
\[
\reg_D^\psi[\widehat{\h}_{S}] \,\,\leq\,\, 2L\xi\E_X\big[\big\|\widehat{\boldeta}_{S'}(X) \,-\, \boldeta(X)\big\|_1\big] \,+\, 
2LC\sqrt{\frac{n^2\log(n)\log(\alpha m) + \log(n^2/\delta)}{\alpha m}},
\]
where $C > 0$ is a distribution-independent constant.
\end{thm}
The above result applies to several performance metrics including the AMS measure ($\xi = 1$) \cite{Kotlowski14}, the binary F-measure ($\xi = 1/\pi_1$) \cite{Parambath+14} and the multi-class micro F-measure ($\xi = 1/(1-\pi_1)$) \cite{Parambath+14}. 
The proof of this theorem follows a similar progression as that of \Thm{thm:consistency-brute-force} and additionally makes use of the convexity-like property of $\psi$ and the following regret bound for a linear/decomposable performance metric defined using a bounded gain matrix.

\begin{lem}[\textbf{Regret Bound for Linear/Decomposable Performance Metric with Bounded Gain Matrix}]
\label{lem:regret-bound-linear-psi-est-cond-prob}
Let $\G \in [-L,L]^{n\times n}$ be a fixed gain matrix. Let $\widehat{\boldeta}: \X \> \Delta_n$ be a class probability estimation model and $\widehat{\h}_{\G}: \X \> \Delta_n$ be a classifier constructed such that $\big[\widehat{\h}_{\G}(x)\big]_i = 1 \text{  if  } i = \overline{\argmax}_{y \in [n]} {\g}_{y}^\top \widehat{\boldeta}(x)$. We then have
\begin{equation*}
\label{eqn:apx-LMO-1}
\max_{\h:\X\>\Delta_n} \langle \G , \conf(\h , D) \rangle \,-\, \langle \G , \conf(\widehat{\h}_\G,D) \rangle \,\,\leq\,\, 2L\E_X\big[\big\|\widehat{\boldeta}(X) \,-\, \boldeta(X)\big\|_1\big].
\end{equation*}
\end{lem}

\begin{rem}[\textbf{Connection to the method of Parambath et al. (2014) \cite{Parambath+14}}]
For certain classes of performance metrics, the brute-force method in Algorithm \ref{algo:brute-force-plug-in-non-decomposable} can be made more efficient by considering in the maximization step, only those gain matrices that are obtained from gradients of $\psi$ at feasible confusion matrices in $\CC_D$. This is beneficial for example, in the case of fractional-linear performance metrics such as the binary and micro F-measure, where any gradient obtained from a feasible confusion matrix can be parametrized using a single scalar. The method of Parambath et al. (2014), which makes use of this fact, can be seen as a special case of Algorithm \ref{algo:brute-force-plug-in-non-decomposable}.
\end{rem}

\section{A Consistent and Efficient Algorithm for Multi-class Non-decomposable Metrics Based on the Conditional Gradient Method}
\label{sec:frank-wolfe}

While the (brute-force) plug-in method analyzed in the previous section is consistent for any non-decomposable metric for which the optimal classifier is of a certain desired form, the number of parameters that need to be tuned in this method grows with the number of classes $n$; in particular, the number of evaluations of the performance metric required in this method could be exponential in $n$. In this section, we provide an alternate efficient learning algorithm based on the conditional gradient (CG) optimization method and show that this algorithm is consistent for a large family of concave performance metrics. Also, unlike the brute-force plug-in, the CG based method makes no assumption on the form of the optimal classifier and hence on the underlying distribution.

More specifically, we pose the problem of learning a classifier for a non-decomposable metric as a constrained optimization problem over the space of feasible confusion matrices, and explore the use of optimization methods for solving this problem. However, unlike a standard optimization problem where the constraint is explicitly specified, in the problem that we consider, testing feasibility of a confusion matrix is not tractable in general; this precludes the use of standard gradient descent based constrained optimization solvers for this problem. Instead, we make use of the conditional gradient (CG) method which does not require the constraint set to be explicitly specified, and instead only requires access to a linear optimization oracle over the constraint set \cite{FrankWolfe56}. In particular, this method proceeds via a sequence of linear optimization steps, each of which is equivalent to maximization of a decomposable performance metric and thus can be solved efficiently.

\begin{figure*}[t]
\begin{algorithm}[H]
\caption{
Idealized Conditional Gradient Algorithm for Multi-class Non-decomposable Performance Metric.}
\label{algo:FW-idealized}
\begin{algorithmic}
\STATE \textbf{Input:} $D$, $\psi: [0, 1]^{n\times n} \> \R_+$
\STATE \textbf{Parameters:} $\kappa \in \N$, $\epsilon > 0$
\STATE Choose an initial classifier $\h^0:\X\>\Delta_n$
\STATE $T = \kappa m$
\STATE \textbf{for} $j =  1$ \textbf{to} $T$ \textbf{do}
\STATE ~~~~~~~~~~$\G^{j}=\nabla\psi(\conf(\h^{j-1}, D))$
\vspace{3pt}
\STATE ~~~~~~~~~~\textbf{Approximate Linear Maximization:}
\STATE ~~~~~~~~~~ ~~~~~~Choose $\u^j:\X\>\Delta_n$ such that $\langle \G^j , \conf(\u^j,D) \rangle \,\geq\, \underset{\u:\,\X \> \Delta_n}{\max}\,\langle \G^j , \conf(\u,D) \rangle -\epsilon$
\vspace{1pt}
\STATE ~~~~~~~~~~Construct ${\h}^{j}:\X\>\Delta_n$ such that ${\h}^{j}(x) = \big(1-\frac{2}{j+1}\big){\h}^{j-1}(x) + \frac{2}{j+1} {\u}^j(x), \,\forall x \in \X$
\STATE \textbf{end for}
\STATE \textbf{Output:} $\h^\FW \equiv \h^T$
\end{algorithmic}
\end{algorithm}
\vspace{-0.5cm}
\end{figure*}

We first present an idealized version of the above CG based learning algorithm, where we assume access to the underlying distribution $D$ (see Algorithm \ref{algo:FW-idealized}). Each iteration of this algorithm maintains a classifier $\h^j$ and (approximately) maximizes a decomposable performance metric given by the gradient of $\psi$ at the confusion matrix for $\h^j$. For a concave and smooth $\psi$, one can derive, by extending the standard CG analysis, the following regret bound guarantee for this algorithm \cite{Jaggi13}; we shall later see how this guarantee can be extended to non-smooth performance metrics.

\begin{thm}[\textbf{Regret Bound for (Idealized) Conditional Gradient Algorithm for Concave Smooth Non-decomposable Metrics}]
\label{thm:regret-bound-fw-idealized}
Let $\psi: [0, 1]^{n\times n} \> \R_+$ be concave over $\CC_D$ and $\beta$-smooth w.r.t. the $\ell_1$-norm  over $\CC_D$ . Let ${\h}^\FW$ be the classifier learned by Algorithm \ref{algo:FW-idealized} with parameters $\kappa \in \N$ and $\epsilon > 0$. Then
\[
\reg_D^\psi[{\h}^\FW] \,\,\leq\,\,2\epsilon \,+\, \frac{8\beta}{\kappa m+2}.
\]
\end{thm} 

Since in practice, one does not have access to $D$, we consider a sample-based version of Algorithm \ref{algo:FW-idealized}, where in each iteration, the gradient for the current classifier is computed using a sample-based estimate of the confusion matrix of the classifier and the solution to the linear maximization step is a plug-in classifier obtained from a suitable class probability estimation model (see Algorithm \ref{algo:BFW}); we shall refer to this method as the `BayesCG' algorithm. Clearly, this algorithm runs in time polynomial in the number of classes $n$ and number of training example $m$.

It is important to note that the BayesCG algorithm essentially mimics the earlier idealized algorithm, with the approximation factor $\epsilon$ in the linear maximization step now depending on the input training sample. Using this observation and the above regret bound guarantee for the idealized algorithm, we now show that the BayesCG algorithm is consistent for any concave smooth performance metric.

\begin{figure*}[t]
\begin{algorithm}[H]
\caption{
\textbf{BayesCG}: Sample-based Conditional Gradient Algorithm for Multi-class Non-decomposable Performance Metric.}
\label{algo:BFW}
\begin{algorithmic}
\STATE \textbf{Input:} $S=((x_1,y_1), \ldots, (x_m, y_m)) \in (\X\times [n])^m$, $\psi: [0, 1]^{n\times n} \> \R_+$
\STATE \textbf{Parameters:} $\alpha\in(0,1)$, $\kappa \in \N$
\STATE Split $S$ into two sets $S'$ and $S''$ with sizes $m_1 = \lfloor (1-\alpha) m \rfloor$ and $m_2 = \lceil \alpha m \rceil$. \\
\STATE Learn $\hat{\boldeta}_{S'}=\CPE(S')$, where $\CPE: \cup_{m=1}^\infty (\X \times [n])^m \> \Delta_n^\X$ is a suitable CPE algorithm 
\STATE Choose an initial classifier $\widehat{\h}^0:\X\>\Delta_n$
\STATE $T = \kappa m$
\STATE \textbf{for} $j =  1$ \textbf{to} $T$ \textbf{do}
\STATE ~~~~~~~~~~$\widehat{\G}^{j}=\nabla\psi(\conf(\widehat{\h}^{j-1},D_{S''}))$
\vspace{3pt}
\STATE ~~~~~~~~~~\textbf{Approximate Linear Maximization:}
\STATE ~~~~~~~~~~ ~~~~~~Construct $\widehat{\u}^j:\X\>\Delta_n$ such that
	$\widehat{u}^j_i(x) = 1$\, if \,$i = \overline{\argmax}_{y\in [n]}\,\,\widehat{\g}_y^{j\top} \hat{\boldeta}_{S'}(x), \,\forall x \in \X$
\STATE ~~~~~~~~~~Construct $\widehat{\h}^{j}:\X\>\Delta_n$ such that $\widehat{\h}^{j}(x) = \big(1-\frac{2}{j+1}\big)\widehat{\h}^{j-1}(x) + \frac{2}{j+1} \widehat{\u}^j(x), \,\forall x \in \X$ 
\STATE \textbf{end for}
\STATE \textbf{Output:} $\widehat{\h}^\FW_{S} \equiv \widehat{\h}^{T}$
\end{algorithmic}
\end{algorithm}
\vspace{-0.6cm}
\end{figure*}

\begin{thm}[\textbf{Consistency of Sample-based Conditional Gradient Algorithm for Concave Smooth Non-decomposable Metrics}]
\label{thm:BFW-consistency}
Let $S = (S', S'') \in (\X \times [n])^m$ be the given training sample drawn i.i.d. from distribution $D$. Let $\psi: [0, 1]^{n\times n} \> \R_+$ be concave over $\CC_D$, $L$-Lipschitz over $\CC_{D_{S''}}$ and $\beta$-smooth, both w.r.t. the $\ell_1$ norm. Let $\widehat{\h}^\FW_S$ be the classifier learned by Algorithm \ref{algo:BFW} using training sample $S$ with parameters $\alpha \in (0,1)$ and  $\kappa \in \N$. Then for any $\delta \in [0,1]$, we have with probability at least $1-\delta$ (over draw of $S$ from $D^m$), 
\[
\reg_D^\psi[\widehat{\h}^\FW_{S}] \,\,\leq\,\, 4L\E_X\big[\big\|\widehat{\boldeta}_{S'}(X) \,-\, \boldeta(X)\big\|_1\big] \,+\, 4\beta n^2C\sqrt{\frac{n^2\log(n)\log(\alpha m) + \log(n^2/\delta)}{\alpha m}} \,+\, \frac{8\beta}{\kappa m+2},
\]
where $C > 0$ is a distribution-independent constant. Thus, if the $\CPE$ algorithm used in Algorithm \ref{algo:BFW} is such that $\E_X \big[ \big\|\widehat{\boldeta}_{S'}(X) - \boldeta(X)\big\|_1\big]  \xrightarrow{P} 0$, then $\reg_D^\psi[\widehat{\h}_{S}] \,\xrightarrow{P} \, 0$ (as $m \rightarrow \infty$). 
\end{thm}

\noindent A key element of the proof of the above theorem is in showing that the BayesCG algorithm solves the CG linear maximization step approximately; this makes use of Lemma \ref{lem:regret-bound-linear-psi-est-cond-prob} and Lemma \ref{lem:uconvg-conf} in the previous section (along with the smoothness assumption on $\psi$).

\begin{lem}[\textbf{Approximation Factor for Linear Maximization Step in Algorithm \ref{algo:BFW}}]
\label{lem:approx-max-oracle}
Let $\psi: [0, 1]^{n\times n} \> \R_+$ satisfy the assumptions in Theorem \ref{thm:BFW-consistency}. Let $\widehat{\u}^j$ and $\widehat{\h}^{j}$ be the classifiers constructed in any given iteration $j$ of Algorithm \ref{algo:BFW} using training sample $S = (S', S'') \in (\X \times [n])^m$ and parameter $\alpha \in (0,1)$. Also, let $\G^j = \nabla \psi(\conf(\widehat{\h}^{j-1}, D))$. Then for any $\delta \in [0,1]$, we have with probability at least $1 - \delta$ (over draw of $S$ from $D^m$) for all $1\leq j\leq T$
\[
\langle {\G}^j , \conf(\widehat{\u}^j,D) \rangle \,\geq\, \underset{\u:\,\X \> \Delta_n}{\max}\,\langle {\G}^j , \conf(\u,D) \rangle \,-\, \widehat{\epsilon}_S,
\]
where
\[
\widehat{\epsilon}_S \,=\, 2L\E_X\big[\big\|\widehat{\boldeta}_{S'}(X) \,-\, \boldeta(X)\big\|_1\big] \,+\, 2\beta n^2C\sqrt{\frac{n^2\log(n)\log(\alpha m) + \log(n^2/\delta)}{\alpha m}},
\]
for a distribution-independent constant $C>0$.
\end{lem}

\begin{proof}[Proof of Theorem \ref{thm:BFW-consistency}] 
 The proof follows from Lemma \ref{lem:approx-max-oracle} and the regret bound in \Thm{thm:regret-bound-fw-idealized}.
\end{proof}

While the consistency result in Theorem \ref{thm:BFW-consistency} applies only to smooth performance metrics, for non-smooth performance metric such as the G-mean metric and several others in Table \ref{tab:perf-metrics}, one can apply Algorithm \ref{algo:BFW} to a suitable smoothed version of the metric (indicated below by $\psi_\rho: [0, 1]^{n\times n} \> \R_+$ for some $\rho \in (0,1)$, with $\lim_{\rho \> 0} \psi_\rho = \psi$), and obtain the following regret bound for the original performance metric.

\begin{thm}[\textbf{Regret Bound for Sample-based Conditional Gradient Algorithm for a Larger Family of Non-decomposable Metrics}]
\label{thm:BFW-consistency-non-smooth}
Let $S = (S', S'') \in (\X \times [n])^m$ be the given training sample drawn i.i.d. from distribution $D$. Let $\psi: [0, 1]^{n\times n} \> \R_+$ be such that for any $\rho \in (0,1)$, there exists $\psi_\rho: [0, 1]^{n\times n} \> \R_+$ which is concave over $\CC_D$, $L_\rho$-Lipschitz w.r.t. the $\ell_1$ norm over $\CC_{D_{S''}}$ and $\beta_\rho$-smooth w.r.t. the $\ell_1$-norm, with 
$$\sup_{\C\,\in\,\CC_D} |\psi(\C) - \psi_\rho(\C)| \leq \theta(\rho),$$ for some strictly increasing function $\theta: \R_+ \> \R_+$. Let $\widehat{\h}^{\FW,\rho}_S$ be the classifier learned by Algorithm \ref{algo:BFW} when applied to $\psi_\rho$ with training sample $S$ and parameters $\kappa \in \N$ and $\alpha \in (0,1)$. Then for any $\delta \in [0,1]$, we have with probability at least $1-\delta$ (over draw of $S$ from $D^m$)
\begin{eqnarray*}
\reg_D^\psi[\widehat{\h}^{\FW,\rho}_{S}] &\leq& 4L_\rho\E_X\big[\big\|\widehat{\boldeta}_{S'}(X) \,-\, \boldeta(X)\big\|_1\big] \,+\,  4\beta_\rho n^2C\sqrt{\frac{n^2\log(n)\log(\alpha m) + \log(n^2/\delta)}{\alpha m}}\\ 
&& \hspace{9.5cm}
+\, \frac{8\beta_\rho}{\kappa m+2} \,+\, 2\theta(\rho),
\end{eqnarray*}
where $C > 0$ is a distribution-independent constant. 
\end{thm}

\subsection{Instantiation to specific concave multi-class performance metrics}
We now instantiate the regret bound in Theorem \ref{thm:BFW-consistency-non-smooth} to several performance metrics in \Tab{tab:perf-metrics} which happen to be concave but non-smooth. \Tab{tab:perf-metrics-smoothed} contains the smoothed version of these performance metrics, along with the resulting Lipschitz and smoothness constant. We then have the following consistency result for these metrics as a corollary of \Thm{thm:BFW-consistency-non-smooth}.

\begin{cor}[\textbf{Consistency of the Sample-based Conditional Gradient Algorithm for H-mean, Q-mean and G-mean}]
Let $S = (S', S'') \in (\X \times [n])^m$ be the given training sample drawn i.i.d. from distribution $D$. For each of H-mean, Q-mean and G-mean, let $\psi_\rho$ be chosen as prescribed in Table \ref{tab:perf-metrics-smoothed}. Let $\widehat{\h}^{\FW,\rho}_S$ be the classifier learned by Algorithm \ref{algo:BFW} when applied to $\psi_\rho$ with training sample $S$. If the $\CPE$ algorithm used in Algorithm \ref{algo:BFW} is such that $\E_X \big[ \big\|\widehat{\boldeta}_{S'}(X) - \boldeta(X)\big\|_1\big]  \xrightarrow{P} 0$, then for each of the above performance metrics, one can choose values of $\rho \> 0$ (as $m \rightarrow \infty$) so that $\reg_D^\psi[\widehat{\h}_{S}] \,\xrightarrow{P} \, 0$ (as $m \rightarrow \infty$).
\end{cor}

\begin{table*}[t]
	\caption{Performance Metrics $\Pf_D[\h] = \psi(\text{conf}(\h, D) \equiv \C)$, for which $\psi: [0, 1]^{n\times n} \> \R_+$ is concave but non-smooth (see \Tab{tab:perf-metrics} for the form of $\psi$ for these metrics). A smoothed version of this function $\psi_\rho: [0, 1]^{n\times n} \> \R_+$ for any $\rho \in (0,1)$ is given in the second column; in each case, $\psi_\rho$ is also concave. The form of $\theta(\rho)$ (defined in \Thm{thm:BFW-consistency-non-smooth}), the Lipschitz constant $L_\rho$ and smoothness parameter $\beta_\rho$ for the smoothed function are given respectively in the third, fourth and fifth columns. Here, we denote $\pi_{\min} = \min_{y \in [n]} \pi_y$. 
Details of all calculations can be found in Appendix \ref{sec:smoothed-metrics}}
	\label{tab:perf-metrics-smoothed}
	\vspace{7pt}
	\centering
	\small
{\scriptsize
\begin{tabular}{@{}lllccc@{}}
		\hline
\vspace{-7pt}\\
		\textbf{metric} &  \multicolumn{1}{c}{$\psi_\rho(\C)$} & $\theta(\rho)$ & $L_\rho$ & $\beta_\rho$ \\
		\hline
		\vspace{-7pt}
		\\
		H-Mean (HM) 
		& ${n}\Big({\sum_{y=1}^n \frac{\sum_{\widehat{y} = 1}^n C_{y,\widehat{y}}\,+\,\rho}{C_{y,y} \,+\, \rho}}\Big)^{-1}$  & $\frac{n}{\pi_{\min}}\rho$ & $\frac{n}{\rho}$ & $\frac{2n}{\rho^2}$ 
\\[7pt]
		Q-Mean (QM) 
		& $1-\sqrt{\frac{1}{n}\sum_{y=1}^n\Big(\frac{\sum_{\widehat{y} \ne y} C_{y,\widehat{y}}\,+\,\rho}{\sum_{\widehat{y} = 1}^n C_{y,\widehat{y}}\,+\,\rho}\Big)^2}$ & $\frac{1}{\pi_{\min}\sqrt{n}}\rho$ & $\frac{1}{\sqrt{n}}\frac{1}{\rho}$ & $\frac{2}{\sqrt{n}}\frac{1}{\rho^2}\Big(1+\frac{1}{\rho}\Big)$
\\[7pt]		
		G-Mean (GM) 
		& $\Big(\prod_{y=1}^n
\Big(\frac{C_{y,y}\,+\,\rho}{\sum_{\widehat{y} = 1}^n C_{y,\widehat{y}}\,+\,\rho}\Big)\Big)^{1/n}$  & $2\rho^{1/n}$& $\frac{1}{n}\frac{1}{\rho}\Big(1+\frac{1}{\rho}\Big)^{1 - 1/n}$ & $\frac{1}{n^2}\frac{1}{\rho^3}\Big(1+\frac{1}{\rho}\Big)^{1 - 2/n}$  
\\[7pt]
		\hline
	\end{tabular} 
}	\vspace{-5pt}
\end{table*}



\begin{rem}[\textbf{BayesCG is consistent for non-continuous distributions}]
While the consistency guarantee for the brute-force plug-in method discussed in Section \ref{sec:plug-in-brute-force} makes crucial use of the form of the optimal classifier for the given performance metric, requiring a continuity assumption on the distribution (Assumption A), the BayesCG method requires no such assumption on the distribution. In particular, when a distribution does not satisfy Assumption A, a randomized classifier can yield a strictly higher performance value than the best deterministic classifier for the given performance metric (e.g., for the distribution described in Remark \ref{rem:necessity-ass-A}, the randomized classifier $\h^*$ yields a strictly higher H-mean value than all deterministic classifiers). Since the brute-force plug-in algorithm learns a deterministic classifier of a specific form, it fails to be consistent for such distributions. On the other hand, the final classifier returned by the BayesCG algorithm is a randomized classifier obtained from an ensemble deterministic classifiers, where the size of this ensemble grows with the number of training examples, thus enabling this method to handle a general distribution that does not satisfy Assumption A.
\end{rem}

\begin{rem}[\textbf{Extension to non-differentiable concave metrics}]
The consistency results that we have seen so far for the BayesCG algorithm have assumed that the given performance metric is concave and differentiable, and hence do not apply to the min-max metric in Table \ref{tab:perf-metrics}, which is (concave, but) not differentiable. It is indeed possible to derive a version of the BayesCG algorithm that is consistent for such continuous concave metrics, by working with a smooth differentiable approximation to these performance metrics \cite{Lan13}. The proof of consistency for the resulting learning algorithm is however slightly more involved, requiring us to deal with approximate gradients to these metrics \cite{Lan13}, and is reserved for a longer version of this paper. 
\end{rem}


\begin{rem}[\textbf{Extension to fractional-linear metrics}]
We would also like to point out that there is a variant of the CG method used in the BayesCG algorithm that can be applied to non-concave optimization objectives \cite{Bertsekas99}, but this method can get stuck in a stationary point that is not a globally optimal solution, and hence the resulting learning algorithm need not be consistent for a general non-concave performance metric. However, one can show (without an explicit regret bound) that this variant of the BayesCG algorithm is consistent for a special class of non-concave performance metrics that are fractional-linear, such as the binary F-measure, the JAC metric and the multi-class micro F-measure, where owing to the pseudo-linear structure of these performance metrics, all stationary points are globally optimal solutions \cite{CambiniMartein08}. 
\end{rem}

\section{Conclusion}
We provide a unified framework for analysing a general non-decomposable multi-class performance metric that cannot be expressed as a sum of losses on individual examples such as the multi-class F-measure and the multi-class G-mean metrics. Using this framework, we give a characterization of the optimal classifier for a general non-decomposable performance metric, subsuming several previous results on binary non-decomposable metrics. We then design a efficient learning algorithm based on the conditional gradient (CG) optimization method that is consistent for a large family of concave performance metrics. Our proof techniques are novel and involve application of tools from the optimization literature, particularly those used in the convergence analysis of the CG method.

~\\
\noindent \textbf{Acknowledgements.}
HGR thanks Ambuj Tewari and Clayton Scott for helpful discussions. HN thanks Prateek Jain and Purushottam Kar for helpful discussions. HGR is supported by a TCS PhD Fellowship. HN is supported by the Google India PhD Fellowship. SA thanks the Department of Science and Technology (DST) of the Government of India for a Ramanujan Fellowship, and the Indo-US Science and Technology Forum (IUSSTF) for their support.

\begin{small}
\bibliographystyle{unsrt}
\bibliography{multi-class-non-decomposable-performance-measures}
\end{small}

\appendix
\section{Proofs}

\subsection{Proof of Proposition \ref{prop:convexity-cc}}
\begin{prop*}[\textbf{Convexity of $\CC_D$}]
$\CC_D$ is a convex set.
\end{prop*}
\begin{proof}
 Let $\C_1,\C_2\in\CC_D$. Let $\lambda\in[0,1]$. We show that $\lambda\C_1 + (1-\lambda)\C_2 \in \CC_D$. 
 
 By definition of $\CC_D$, there exists randomized classifiers $\h_1,\h_2:\X\>\Delta_n$ such that
 \begin{eqnarray*}
  \C_1 = \conf(\h_1,D) \\
  \C_2 = \conf(\h_2,D) \\
 \end{eqnarray*}
Consider the randomized classifier $h^\lambda:\X\>\Delta_n$ defined as
$$\h^\lambda(x)= \lambda \h_1(x) + (1-\lambda)\h_2(x)~.$$
It can be seen that
$$\conf(\h^\lambda,D) = \lambda \C_1 + (1-\lambda)\C_2 ~.$$
\end{proof}

\subsection{Proof of Lemma \ref{lem:opt-in-closure-same-as-opt}}
While Lemma \ref{lem:opt-in-closure-same-as-opt} is simple to state, its proof is rather intricate and hence we give its proof via several intermediate lemmas and propositions.

\begin{lem}[\textbf{Confusion matrix as an integration}]
\label{lem:conf-as-integration}
 Let $\f:\Delta_n\>\Delta_n$. Then
 $$\conf(\f\circ{\boldeta},D) = \int_{\p \in \Delta_n} \p (\f(\p))^\top d\nu(\p)\;.$$
\end{lem}
\begin{proof}
 \begin{eqnarray*}
  \big[\conf(\f\circ{\boldeta},D)\big]_{i,j} 
  &=& 
  \E_{(X,Y)\sim D} [ f_j(\boldeta(X)) \cdot \1(Y=i) ]  \\
  &=&
  \E_{\p\sim\nu} \E_{(X,Y)\sim D} \left[ f_j(\p) \cdot \1(Y=i) \big|  \boldeta(X)=\p \right]  \\
  &=&
  \E_{\p\sim\nu} \big[p_i f_j(\p)\big]
 \end{eqnarray*}

\end{proof}

\begin{prop}[\textbf{Sufficiency of conditional probability}]
\label{prop:equivalent-cond-prob-classifier}
Let $D$ be a distribution over $\X\times\Y$. For any randomized classifier $\h:\X\>\Delta_n$ there exists another randomized classifier $\h':\X\>\Delta_n$ such that $\conf(\h,D)=\conf(\h',D)$ and $\h'$ is such that 
$\h'(x) = \f({\boldeta}(x))$, for some $\f:\Delta_n\>\Delta_n$.
\end{prop}
\begin{proof}
 Let $\h:\X\>\Delta_n$. Define $\f:\Delta_n\>\Delta_n$ as follows,
 \begin{eqnarray*}
  \f(\p) = \E_{X\sim D_\X} [\h(X) | \boldeta(X)=\p]\;.
 \end{eqnarray*}
 We then have for any $i,j\in[n]$ that,
 \begin{eqnarray*}
  \big[ \conf(\h,D) \big]_{i,j} 
  &=& 
  \E_{(X,Y)\sim D} [ h_j(X) \cdot\1(Y=i) ]  \\ 
  &=&
  \E_{\p\sim \nu} \E_{(X,Y)\sim D} [ h_j(X) \cdot\1(Y=i)| \boldeta(X)=\p ] \\
  &=&
  \E_{\p\sim \nu} \left[ \E_{(X,Y)\sim D} [ h_j(X) | \boldeta(X)=\p ] \cdot  \E_{(X,Y)\sim D} [ \cdot\1(Y=i)| \boldeta(X)=\p ] \right] \\ 
  &=& 
  \E_{\p\sim \nu} \big[ f_j(\p) p_i \big] \\
  &=&
  \big[\conf(\f\circ{\boldeta},D)\big]_{i,j} 
 \end{eqnarray*}
where the third equality follows because, given $\boldeta(X)$, the random variables $X$ and $Y$ are independent.
\end{proof}

\begin{lem}[\textbf{Continuity of the $\conf$ mapping}]
\label{lem:conf-continuity}
 Let $D$ be a distribution over $\X\times \Y$. 
Let $\f_1,\f_2:\Delta_n\>\Delta_n$. Then
$$ \big|\big|\conf(\f_1\circ{\boldeta},D) - \conf(\f_2\circ{\boldeta},D) \big|\big|_1  
\leq 
\int_{\p \in \Delta_n}  ||\f_1(\p) - \f_2(\p)||_1 d\nu(\p)  \;.$$
\end{lem}
\begin{proof}
 Let $\f_1,\f_2:\Delta_n\>\Delta_n$
\begin{eqnarray*}
\conf(\f_1\circ{\boldeta},D) - \conf(\f_2\circ{\boldeta},D) 
&=&
\int_{\p \in \Delta_n} \p (\f_1(\p) - \f_2(\p))^\top d\nu(\p) \\
\big|\big|\conf(\f_1\circ{\boldeta},D) - \conf(\f_2\circ{\boldeta},D) \big|\big|_1  
&\leq&
\int_{\p \in \Delta_n} || \p (\f_1(\p) - \f_2(\p))^\top||_1 d\nu(\p) \\
&=&
\int_{\p \in \Delta_n} || \p ||_1 ||\f_1(\p) - \f_2(\p)||_1 d\nu(\p) \\
&=&
\int_{\p \in \Delta_n}  ||\f_1(\p) - \f_2(\p)||_1 d\nu(\p)  \\
\end{eqnarray*}
\end{proof}

\begin{lem}
\label{lem:volume-invese-linear-map}
 Let $d>0$ be any integer. Let $\V\subseteq\R^d$ be compact and convex. Let $f:\R^d\>\R$ be an affine function such that it is non-constant over $\V$. Let $V$ be a vector valued random variable taking values uniformly over $\V$. There exists a constant $\alpha>0$ such that for all $c\in\R$ and $\epsilon\in\R_+$ we have
 $$ \P(f(V)\in[c,c+\epsilon]) \leq \alpha \epsilon \;.$$ 
\end{lem}
\begin{proof}
Let us assume for now that affine hull of $\V$ is the entire space $\R^d$. 

For any integer $i$ and set $\cA$, let $\vol_i(\cA)$ denote the $i$-th dimensional volume of the set $\cA$. Note that $\vol_i(\cA)$ is undefined if the affine-hull dimension of $\cA$ is greater than $i$ and is equal to zero if the affine-hull dimension of $\cA$ is lesser than $i$. 

For any $r>0$ and any integer $i>0$ let $B_i(r)\subseteq\R^i$ denote the set 
$B_i(r)=\{\x\in\R^i:||\x||_2 \leq r \} \;.$ Also let $R$ be the smallest value such that $\V\subseteq B_d(R)$. 

Let the affine function $f$ be such that for all $\x\in\R^d$, the value $f(\x)=\g^\top \x + u$. By the assumption of non-constancy of $f$ on $\V$ we have that $\g\neq 0$.

We now have that 
\begin{eqnarray*}
 \P(f(V)\in[c,c+\epsilon]) 
 &=& 
 \frac{\vol_d\left( \{\v\in\V: c-u \leq \g^\top \v \leq c-u+\epsilon \right)}{\vol_d(\V)} \\
 &\leq&
 \frac{\vol_d\left( \{\v\in B_d(R): c-u \leq \g^\top \v \leq c-u+\epsilon \right)}{\vol_d(\V)} \\
 &\leq&
 \epsilon \cdot \frac{  \vol_{d-1}\big( B_{d-1}(R) \big)}{\vol_d(\V) ||\g||_2} \;.
\end{eqnarray*}
The last inequality follows from the observation that $d$-volume of a strip of a $d$ dimensional sphere of radius $r$ is at most the $d-1$ volume of a $d-1$ dimensional sphere of radius $r$ times the width of the strip, and the width of the strip under consideration here is simply $\frac{\epsilon}{||\g||_{{}_2}}$.

Finally, if the affine hull of $\V$ is not the entire space $\R^d$, one can simply consider the affine-hull of $\V$ to be the entire space and all the above arguments hold with some affine transformations and a smaller $d$.
\end{proof}

\begin{lem}
 \label{lem:volume-trig-prob-boundary}
 Let $D$ be a distribution over $\X\times \Y$. Let $\G\in\R^{n\times n}$ be such that no two columns are identical. Let the measure over conditional probabilities $\nu$, be absolutely continuous w.r.t. the base measure $\mu$. Let $c\geq0$. Let $\cA_c\subseteq\Delta_n$ be the set
 $$\mathcal{A}_c=\{\p\in\Delta_n: (\p^\top \G)_{(1)} - (\p^\top \G)_{(2)} \leq c \} \;,$$
 where for any vector $\v\in\R^n$ and integer $i\in[n]$, the scalar $(\v)_{(i)}$  denotes the $i^{th}$ element among the components of $\v$, when they are arranged in descending order.
 Let $r:\R_+\>\R_+$ be the function defined as
 $$r(c) = \nu(\cA_c) \;.$$
 Then 
 \begin{enumerate}
  \item[(a)] $r$ is a monotonically increasing function.
  \item[(b)] There exists a $C>0$ such that $r$ is a continuous function over $[0,C]$.
  \item[(c)] $r(0)=0$.
 \end{enumerate}
\end{lem}
\begin{proof}
\emph{Part (a):}

The fact that $r$ is a monotonically increasing function is immediately obvious from the observation that $\cA_a\subseteq\cA_b$ for any $a<b$. 

\emph{Part (b):}

Let 
$$C=\frac{1}{2}\min \{d\in\R: \g_{y}-\g_{y'} = d \e \text{ for some }y,y'\in[n], y\neq y'\} \;,$$
where $\e$ is the all ones vector. If there exists no $y,y'$ such that $\g_{y}-\g_{y'}$ is a scalar multiple of $\e$, then we simply set $C=\infty$. Note that by our assumption on unequal columns on $\G$, we always have $C>0$.

For any $c>0$ and $y,y'\in[n]$ with $y\neq y'$, define the set $\cA^{y,y'}_c$ as
\begin{eqnarray*}
  \cA^{y,y'}_c=\{\p\in\Delta_n: \p^\top \g_y - \p^\top \g_{y'}  \leq c \} \;.
\end{eqnarray*}

 For any $c, \epsilon>0$, it can be clearly seen that
\begin{eqnarray*}
\nu(\cA_{c+\epsilon}) - \nu(\cA_c)  
&=&  
\nu(\cA_{c+\epsilon}\setminus\cA_c) \;,\\
\cA_{c+\epsilon}\setminus \cA_{c} 
&\subseteq& 
\bigcup_{y,y'\in[n],y\neq y'}  \Big( \cA^{y,y'}_{c+\epsilon}\setminus \cA^{y,y'}_{c} \Big) \;, \\
\nu(\cA_{c+\epsilon}\setminus \cA_{c} )
&\leq&
\sum_{y,y'\in[n],y\neq y'} \nu \Big( \cA^{y,y'}_{c+\epsilon}\setminus \cA^{y,y'}_{c} \Big) \;.
\end{eqnarray*}
Hence, our proof for continuity of $r$ would be complete, if we show that $\nu\Big(\cA^{y,y'}_{c+\epsilon}\setminus \cA^{y,y'}_{c}\Big)$ goes to zero as $\epsilon$ goes to zero for all $y\neq y'$ and $c\in [0,C]$. 

Let $c\in[0,C]$ and $y,y'\in[n]$ with $y\neq y'$
\begin{eqnarray*}
 \cA^{y,y'}_{c+\epsilon}\setminus \cA^{y,y'}_{c} 
 &=&  
 \{\p\in\Delta_n: c< \p^\top (\g_y-\g_{y'}) \leq c+\epsilon \}\;.
\end{eqnarray*}
If $\g_y-\g_{y'}=d\e$ for some $d$, we have that $\p^\top (\g_y-\g_{y'})=d$ and $d>C$ by definition of $C$. Hence for  small enough $\epsilon$ the set $\cA^{y,y'}_{c+\epsilon}\setminus \cA^{y,y'}_{c} $ is empty.

If $\g_y-\g_{y'}$ is not a scalar multiple of $\e$, then $\p^\top (\g_y-\g_{y'})$ is a non-constant linear function of $\p$ over $\Delta_n$. From Lemma \ref{lem:volume-invese-linear-map}, $\mu\Big(\cA^{y,y'}_{c+\epsilon}\setminus \cA^{y,y'}_{c}\Big)$ goes to zero as $\epsilon$ goes to zero.
And by the absolute continuity of $\nu$ w.r.t. $\mu$, we have $\nu\Big(\cA^{y,y'}_{c+\epsilon}\setminus \cA^{y,y'}_{c}\Big)$ goes to zero as $\epsilon$ goes to zero. 

As the above arguments hold for any $c\in[0,C]$ and $y,y'\in[n]$ with $y\neq y'$, the proof of part (b) is complete.

\emph{Part (c):}

We have,
\begin{eqnarray*}
 \cA_0 
 &\subseteq& 
 \bigcup_{y,y'\in[n],y\neq y'}  \Big( \cA^{y,y'}_{0} \cap \cA^{y',y}_{0} \Big) \;. \\
\end{eqnarray*}
To show $r(0)=0$, we show $\mu\Big( \cA^{y,y'}_{0} \cap \cA^{y',y}_{0} \Big)=0$ for all $y\neq y'$. Let $y,y'\in[n]$ with $y\neq y'$, then
\begin{eqnarray*}
 \Big( \cA^{y,y'}_{0} \cap \cA^{y',y}_{0} \Big) 
 &=&
 \{\p\in\Delta_n: \p^\top (\g_y-\g_{y'}) = 0 \}\;.
\end{eqnarray*}
If $\g_y-\g_{y'}=d\e$ for some $d\neq 0$, the above set is clearly empty. If $\g_y-\g_{y'}$ is not a scalar multiple of $\e$, then $\p^\top (\g_y-\g_{y'})$ is a non-constant linear function of $\p$ over $\Delta_n$, and hence  by Lemma \ref{lem:volume-invese-linear-map}, we have that $\mu\Big( \cA^{y,y'}_{0} \cap \cA^{y',y}_{0} \Big)=0$. By the absolute continuity of $\nu$ w.r.t. $\mu$ we have that $\nu\Big( \cA^{y,y'}_{0} \cap \cA^{y',y}_{0} \Big)=0$.

As the above arguments hold for any $y,y'\in[n]$ with $y\neq y'$, the proof of part (c) is complete.
\end{proof}


\begin{lem}[\textbf{Uniqueness of Optimal Confusion Matrix for Special Gain Matrices}]
\label{lem:opt-in-closure-same-as-opt-general}
Let $D$ be a distribution over $\X\times\Y$.
Let $\nu$ be absolutely continuous w.r.t. $\mu$, let $\G\in\R^{n\times n}$ be such that no two columns are identical. Then,
\[
\argmax_{\C\in\overline{\CC_D}}\, \langle \G, \C \rangle \,=\, \argmax_{\C\in\CC_D} \langle \G, \C \rangle.
\]
Moreover, the above set is a singleton.
\end{lem}
\begin{proof}
We shall proceed by showing that the maximizer of $\langle \G,\C \rangle$ over $\CC_D$ is unique and then show that there exists no other maximizer of  $\langle \G,\C \rangle$ over $\overline{\CC_D}$ .

Using Proposition \ref{prop:equivalent-cond-prob-classifier}, we will only consider classifiers $\h:\X\>\Delta_n$ that can be be decomposed as $\h=\f\circ\boldeta$ for some $\f:\Delta_n\>\Delta_n$.

From Equation \ref{eqn:Bayes-decomposable}, we have that any $\f^*\in \argmax_{\f:\Delta_n\>\Delta_n} \langle \G, \conf(\f\circ\boldeta,D) \rangle$ is such that the following holds $\nu$- almost everywhere
\begin{eqnarray*}
f^*_i(\p) > 0  ~~\text{ only if }~~ i \in \argmax_{y \in [n]} \g_{y}^\top \p  \;.
\end{eqnarray*}
We will show that the maximizer of $\langle \G,\C \rangle$ over $\CC_D$ is unique, simply by showing that any $\f^*$ satisfying the above equation has the same $\conf(\f^*\circ\boldeta,D)$, which we in turn show by proving that any two functions $\f^*$ satisfying the above condition is the same $\nu$ almost everywhere.

For a given $\p\in\Delta_n$, if  $T_\p\=\argmax_{y \in [n]} \g_{y}^\top \p$ is a singleton, then $\f^*(\p)$ is uniquely defined due to the sum to one constraint. If $\p$ is such that $|T_\p|>1$, then $(\p^\top \G)_{(1)} - (\p^\top \G)_{(2)}=0$. From Lemma \ref{lem:volume-trig-prob-boundary}, the $\nu$-measure of all such $\p$ vectors is exactly equal to $r(0)=0$. 

This completes the proof of the uniqueness of the maximizer of $\langle \G,\C \rangle$ over $\CC_D$. Let us denote it by $\C^*$. Also let $\f^*:\Delta_n\>\Delta_n$ with $\C^*=\conf(\f^*\circ\boldeta,D)$ refer to the following fixed function:
$$f^*_i(\p) = \begin{cases}
               1 &\text{ if }i=\overline{\argmax}_{y\in[n]} \p^\top \g_y \\
               0 & \text{ otherwise}
              \end{cases}\;.
$$

Let $\C'\in \argmax_{\C\in\overline{\CC_D}}\, \langle \G, \C \rangle$. Let us assume $\C^*\neq\C'$.  
$$||\C'-\C^*||_1 = \sum_{i=1}^n \sum_{j=1}^n |C'_{i,j} - C^*_{i,j}| = \gamma > 0 \;.$$

We shall go on to derive a contradiction as follows. By virtue of $\C'\in\overline{\CC_D}$ there exists a sequence of classifiers whose confusion matrices approach $\C'$. And hence these classifiers are all `close' to maximal for the gain matrix $\G$. We then show that these classifiers perform strictly worse than $\h^*$ by exploiting that the confusion matrices of these classifiers are bounded away from $\C^*$. This provides us the required contradiction.

As $\C'\in\overline{\CC_D}$, we have that for all $\epsilon>0$, there exists $\C_\epsilon\in\CC_D$, such that $||\C_\epsilon-\C'||_1\leq \epsilon$. This implies that 
\begin{eqnarray}
 ||\C_\epsilon - \C^*||_1 
 &\geq& 
 \gamma-\epsilon \;, \label{eqn:lemma-1}  \\
 \langle \G,\C_\epsilon \rangle \geq \langle \G,\C' \rangle - ||\G||_\infty \epsilon 
 &\geq&
 \langle \G,\C^* \rangle - ||\G||_\infty \epsilon \;. \label{eqn:lemma-2}
 \end{eqnarray}
 
Let $\f_\epsilon:\Delta_n\>\Delta_n$ be s.t. $\C_\epsilon=\conf(\f_\epsilon\circ{\boldeta},D)$.  Let $\mathcal B=\{\p\in\Delta_n:||\f^*(\p) - \f_\epsilon(\p)||_1\geq \frac{\gamma}{4}\}$. Applying Equation \ref{eqn:lemma-1} and Lemma \ref{lem:conf-continuity}  we have
\begin{eqnarray}
 \gamma-\epsilon 
 &\leq& 
 \big|\big| \conf(\f^*\circ{\boldeta^D},D) - \conf(\f_\epsilon\circ{\boldeta^D},D) \big|\big|_1   \nonumber\\
 &\leq&
 \int_{\p \in \Delta_n}  ||\f^*(\p) - \f_\epsilon(\p)||_1 d\nu(\p) \nonumber \\
 &\leq&
 \int_{\p \in \mathcal B} 2   d\nu(\p) + \int_{\p \notin \mathcal B} \frac{\gamma}{4}   d\nu(\p) \nonumber \\
 &=&
 2\nu(\mathcal B) + \frac{\gamma}{4} (1-\nu(\mathcal B)) \nonumber \\
 &\leq&
 2\nu(\mathcal B) + \frac{\gamma}{4} \nonumber \\
 \nu(\mathcal B)
 &\geq&
 \frac{3\gamma}{8} -\frac{\epsilon}{2} \label{eqn:lemma-3}
\end{eqnarray}

For any $c>0$, define $\mathcal{A}_c\subseteq\Delta_n$ as
$$\mathcal{A}_c=\{\p\in\Delta_n: (\p^\top \G)_{(1)} - (\p^\top \G)_{(2)} \leq c \} \;,$$
From Lemma \ref{lem:volume-trig-prob-boundary} we have that $\nu(\cA_c)$ is a continuous function of $c$ close to $0$ and $\nu(\cA_0)=0$. Let  $c>0$ be such that 
\begin{eqnarray}
\nu(\mathcal A_c)\leq\frac{\gamma}{16} \label{eqn:lemma-4}\;. 
\end{eqnarray}
From Equations \ref{eqn:lemma-3} and \ref{eqn:lemma-4}, we have $\nu(\mathcal B\setminus \mathcal A_c)\geq \frac{5\gamma}{16}-\frac{\epsilon}{2}$. Any $\p\in\cB\setminus\cA_c$ is such that 
$$(\p^\top \G)_{(1)} - (\p^\top \G)_{(2)} > c \hspace{1em}\text{and}\hspace{1em} ||\f^*(\p) - \f_\epsilon(\p)||_1\geq \frac{\gamma}{4} \;.$$

For any $\p\in\Delta_n$, we have $\f^*(\p)$, has a $1$ corresponding to the maximum value of $\p^\top\G$ and zero elsewhere. For any $\p\in\cB\setminus\cA_c$, we have $||\f^*(\p) - \f_\epsilon(\p)||_1\geq \frac{\gamma}{4}$, and hence the value of $\f_\epsilon(\p)$ corresponding to the index of maximum value of $\p^\top\G$ is at most $(1-\frac{\gamma}{8})$. In particular, we have
\begin{eqnarray}
 \p^\top\G \f_\epsilon(\p) \leq \left(1-\frac{\gamma}{8}\right) (\p^\top\G)_{(1)} + \left(\frac{\gamma}{8}\right) (\p^\top\G)_{(2)} \;.\label{eqn:lemma-5}
\end{eqnarray}
Thus we have,
\begin{eqnarray*}
   \langle \G,\C^* \rangle - \langle \G,\C_\epsilon \rangle  
 &=&
 \int_{\p\in\Delta_n} \p^\top \G (\f^*(\p) - \f_\epsilon(\p)) d\nu^D(\p)\\
 &=&
 \int_{\p\in\cB\setminus\cA_c} \p^\top \G (\f^*(\p) - \f_\epsilon(\p)) d\nu^D(\p)
 +\int_{\p\in\Delta_n\setminus(\cB\setminus\cA_c)} \p^\top \G (\f^*(\p) - \f_\epsilon(\p)) d\nu^D(\p)\\
 &\geq&
 \int_{\p\in\cB\setminus\cA_c} \p^\top \G (\f^*(\p) - \f_\epsilon(\p)) d\nu^D(\p)\\
 &=& 
 \int_{\p\in\cB\setminus\cA_c}\big( (\p^\top \G)_{(1)}  - \p^\top\G \f_\epsilon(\p)\big) d\nu^D(\p)\\
  &\geq&
 \int_{\p\in\cB\setminus\cA_c}\left( (\p^\top \G)_{(1)}  -  \left(1-\frac{\gamma}{8}\right) (\p^\top\G)_{(1)} - \left(\frac{\gamma}{8}\right) (\p^\top\G)_{(2)}  \right) d\nu^D(\p)\\
&=&
\int_{\p\in\cB\setminus\cA_c} \frac{\gamma}{8} \left ( (\p^\top \G)_{(1)} - (\p^\top \G)_{(2)} \right)  d\nu^D(\p) \\
&\geq&
\frac{\gamma c}{8} \left( \frac{5\gamma}{16}-\frac{\epsilon}{2} \right)
\end{eqnarray*}
If $\epsilon\leq\frac{\gamma}{2}$, we have 
$$ \langle \G,\C^* \rangle - \langle \G,\C_\epsilon \rangle  \geq \frac{\gamma^2 c}{128} \;.$$ 
The above holds for any $\epsilon\in(0,\frac{\gamma}{2}]$, and both $\gamma$ and $c$ do not depend on $\epsilon$. For small enough $\epsilon$, this contradicts Equation \ref{eqn:lemma-2}. We thus have a contradiction for our assumption $\C^*\neq \C'$. 
\end{proof}

The proof of Lemma \ref{lem:opt-in-closure-same-as-opt} simply follows from Lemma \ref{lem:opt-in-closure-same-as-opt-general} by observing that if $\psi$ satisfies Assumption B, then no two columns of its gradient at any point are identical.

\subsection{Proof of Lemma \ref{lem:fixed-gain-matrix-conf}}
\begin{lem*}[\textbf{Convergence of $\conf$ for fixed gain matrix}]
Let $D$ satisfy Assumption A. Let $\widehat{\boldeta}_{\tilde{S}}: \X \> \Delta_n$ be a class probability estimation model learned using a sample $\tilde{S}$ drawn i.i.d. from $D^{\tilde{m}}$. For a fixed gain matrix $\G \in [0,1]^{n \times n}$ such that no two columns are identical, let $\h_{\G}: \X \> \Delta_n$ and $\widehat{\h}_{\G}: \X \> \Delta_n$ be classifiers constructed as follows: $\big[{\h}_{\G}(x)\big]_i = 1 \text{ if  } i = \overline{\argmax}_{y \in [n]} {\g}_{y}^\top {\boldeta}(x), \,\forall x \in \X$ and $\big[\widehat{\h}_{\G}(x)\big]_i = 1 \text{ if  } i = \overline{\argmax}_{y \in [n]} {\g}_{y}^\top \widehat{\boldeta}_{\tilde{S}}(x), \,\forall x \in \X$. If $\widehat{\boldeta}_{\tilde{S}}$ is such that $\E_X \big[ \big\|\widehat{\boldeta}_{\tilde{S}}(X) - \boldeta(X)\big\|_1\big] \xrightarrow{P} 0$, then $\forall i,j,\,\big[\conf(\widehat{\h}_{\G}, D)\big]_{ij} \,\xrightarrow{P} \, \big[\conf({\h}_{\G}, D)\big]_{ij}$ (as $m \rightarrow \infty$).
\end{lem*}
\begin{proof}
 Fix $\epsilon,\delta,\delta'>0$. By virtue of $\E_X \big[ \big\|\widehat{\boldeta}_{\tilde{S}}(X) - \boldeta(X)\big\|_1\big] \xrightarrow{P} 0$, there exists a $\tilde M_{\epsilon,\delta}$, such that for all $\tilde m > \tilde M_{\epsilon,\delta}$ we have with probability  at least $1-\delta$ over the draw of $\tilde S$ that
 $$\E_X \big[ \big\|\widehat{\boldeta}_{\tilde{S}}(X) - \boldeta(X)\big\|_1\big] < \epsilon \;.$$
 Let $\tilde m >\tilde M_{\epsilon,\delta}$. By Markov's inequality we have that
 $$\P_X\left(\big\|\widehat{\boldeta}_{\tilde{S}}(X) - \boldeta(X)\big\|_1 > \frac{\E_X \big[ \big\|\widehat{\boldeta}_{\tilde{S}}(X) - \boldeta(X)\big\|_1\big]}{\delta'}\right) \leq \delta' \;.$$ 
 Hence with probability  at least $1-\delta-\delta'$ over the draw of both $\tilde S$ and $X$, we have
 \begin{eqnarray}
 \label{eqn:eta-hat-eta-close}
 \big\|\widehat{\boldeta}_{\tilde{S}}(X) - \boldeta(X)\big\|_1 \leq \frac{\epsilon}{\delta'} \;. 
 \end{eqnarray}

Based on the above inequality we will argue that ${\h}_{\G}$ and $\widehat{\h}_{\G}$ have the same value for most instances.
 
For any $x\in\X$. Let $y^*(x)=\overline\argmax_{y\in[n]}\g_y^\top \boldeta(x)$ and $\hat{y}^*(x)=\overline\argmax_{y\in[n]}\g_y^\top \widehat{\boldeta}_{\tilde{S}}(x)$.  The following implications hold:
 \begin{eqnarray*}
  {\h}_{\G}(X) \neq \widehat{\h}_{\G}(X)
  &\Rightarrow&
  y^*(X) \neq \hat{y}^*(X) \\
  &\Rightarrow&
  \g_{y^*(X)}^\top \boldeta(X) > \g_{\hat{y}^*(X)}^\top \boldeta(X) \text{ and } \g_{y^*(X)}^\top \widehat{\boldeta}_{\tilde{S}}(X) < \g_{\hat{y}^*(X)}^\top \widehat{\boldeta}_{\tilde{S}}(X)
 \end{eqnarray*}
Using equation \ref{eqn:eta-hat-eta-close} the following  holds with probability  at least $1-\delta-\delta'$ over $X$ and $\tilde S$:
\begin{eqnarray*}
 {\h}_{\G}(X) \neq \widehat{\h}_{\G}(X)
 &\Rightarrow&
 \g_{\hat{y}^*(X)}^\top \boldeta(X) <  \g_{y^*(X)}^\top \boldeta(X) <  \g_{\hat{y}^*(X)}^\top \boldeta(X) +2\frac{\epsilon}{\delta'} \\
 &\Rightarrow&
 \left( \g_{y^*(X)} - \g_{\hat{y}^*(X)} \right) ^\top \boldeta(X) \in \left[0,\frac{2\epsilon}{\delta'}\right] \\
 &\Rightarrow&
 \exists y,y'\in[n], y\neq y' \text{ s.t. } \left( \g_{y} - \g_{y'} \right) ^\top \boldeta(X) \in \left[0,\frac{2\epsilon}{\delta'}\right] 
\end{eqnarray*}

For any $y,y'\in[n]$ with $y\neq y'$  define the set $\A_{y,y'}\subseteq\Delta_n$ as
$$\cA_{y,y'}=\{\p\in\Delta_n: \left( \g_{y} - \g_{y'} \right) ^\top \p \in \left[0,2\epsilon/\delta'\right] \} \;$$

We thus have that ${\h}_{\G}(X) = \widehat{\h}_{\G}(X)$ with probability at least $1-\delta-\delta'-\sum_{y,y'\in[n], y\neq y'} \nu(\cA_{y,y'})$. 

As $\G$ has no two identical columns we have that $\left( \g_{y} - \g_{y'} \right)^\top \p$ is never zero for all $\p\in\Delta_n$. Let $y,y'\in[n]$ with $y\neq y'$. If $\g_{y}-\g_{y'}=d\e$ for some  and $d>0$, we have that $\cA_{y,y'}$ is empty for small enough $\frac{\epsilon}{\delta'}$. Otherwise, we have by Lemma \ref{lem:volume-invese-linear-map} that $\mu(\cA_{y,y'})$ approaches $0$ as $\frac{\epsilon}{\delta'}$ approaches $0$. And by the absolute continuity of $\nu$ w.r.t. $\mu$ , we have that $\nu(\cA_{y,y'})$ also approaches $0$ as $\frac{\epsilon}{\delta'}$ approaches $0$.

Thus by having $\epsilon,\delta,\delta'$ and $\frac{\epsilon}{\delta'}$ simultaneously approach zero, we have that the probability of the statement  ${\h}_{\G}(X) = \widehat{\h}_{\G}(X)$ approaches $1$. And hence $\forall i,j,\,\big[\conf(\widehat{\h}_{\G}, D)\big]_{ij} \,\xrightarrow{P} \, \big[\conf({\h}_{\G}, D)\big]_{ij}$.

\end{proof}

\subsection{Proof of Lemma \ref{lem:uconvg-conf} }
\begin{lem*}[\textbf{Uniform Convergence Generalization Bound for $\conf$ Over $\H_{{\bmu}}$}]
Let ${\bmu}: \X \> \R^n$ be a fixed function and ${\tilde{S}} \in (\X \times [n])^{\tilde{m}}$ be a sample drawn i.i.d. according to $D^{\tilde{m}}$. For any $\delta \in [0,1]$, we have with probability at least $1-\delta$ (over draw of $\tilde{S}$ from $D^{\tilde{m}}$),
\[
\sup_{\h \,\in\, \H_{{\bmu}}}\,\big\|\conf(\h, D) \,-\, \conf(\h, D_{\tilde{S}})\big\|_\infty \,\,\leq\,\, 
C\sqrt{\frac{n^2\log(n)\log(\tilde m) + \log(n^2/\delta)}{\tilde m}},
\]
where $C > 0$ is a distribution-independent constant.
\end{lem*}
\begin{proof}
First observe that every function $\h\in\H_\bmu$ is such that for all $\x\in\X$, the vector $\h(\x)$ is always one of the co-ordinate vectors in $\R^n$.
For any  $a,b\in[n]$ we have,
\begin{eqnarray*}
\sup_{\h\in\H_\bmu}\left| [\conf(\h,D_{\tilde S})]_{a,b} - [\conf(\h,D)]_{a,b} \right| 
&=&
\sup_{\h\in\H_\bmu}\left| \frac{1}{\tilde m}\sum_{i=1}^{\tilde m} \left(\1(y_i=a, h_b(x_i)=1) - \E [\1(Y=a,h_b(X)=1)] \right) \right| \\
&=&
\sup_{h\in\H^b_\bmu}\left| \frac{1}{\tilde m}\sum_{i=1}^{\tilde m} \left(\1(y_i=a, h(x_i)=1) - \E [\1(Y=a,h(X)=1)] \right) \right|\;, \\
\end{eqnarray*}
where $\H^b_\bmu=\{h:\X\>\{0,1\}: \exists \G\in \R^{n\times n}, \forall \x\in\X, h(\x)=\1(b= \argmax_{t\in[n]}\g_t^\top \bmu(\x) ) \}$. The set $\H^b_{\bmu}$ can be seen as hypothesis class whose concepts are the intersection of $n$ halfspaces in $\R^n$ (corresponding to $\bmu(\x)$) through the origin.  Hence we have from Lemma 3.2.3 of Blumer et al. (1989) \cite{Blumer+89} that the VC-dimension of $\H^b_{\bmu}$ is at most $2n^2\log (3n)$. From standard uniform convergence arguments we have that the following holds with probability $1-\delta$,
\begin{eqnarray*}
\sup_{\h\in\H_\bmu}\left| [\conf(\h,D_{\tilde S})]_{a,b} - [\conf(\h,D)]_{a,b} \right| 
\leq 
C \sqrt{\frac{n^2\log(n)\log(\tilde m) + \log(\frac{1}{\delta})}{\tilde m} }
\end{eqnarray*}
where $C>0$ is some  constant.
Applying union bound for all $a,b \in [n]$ we have that the following holds with probability $1-\delta$
\begin{eqnarray*}
\sup_{\h\in\H_\bmu}\left|\left| [\conf(\h,D_{\tilde S})] - [\conf(\h,D)] \right|\right|_\infty
\leq
C \sqrt{\frac{n^2\log(n)\log(\tilde m) + \log(\frac{n^2}{\delta})}{\tilde m} }
\end{eqnarray*}

\end{proof}

\subsection{Proof of Lemma \ref{lem:regret-bound-linear-psi-est-cond-prob}}
\begin{lem*}[\textbf{Regret Bound for Linear/Decomposable Performance Metric with Bounded Gain Matrix}]
Let $\G \in [-L,L]^{n\times n}$ be a fixed gain matrix. Let $\widehat{\boldeta}: \X \> \Delta_n$ be a class probability estimation model and $\widehat{\h}_{\G}: \X \> \Delta_n$ be a classifier constructed such that $\big[\widehat{\h}_{\G}(x)\big]_i = 1 \text{  if  } i = \overline{\argmax}_{y \in [n]} {\g}_{y}^\top \widehat{\boldeta}(x)$. We then have
\begin{equation*}
\max_{\h:\X\>\Delta_n} \langle \G , \conf(\h , D) \rangle \,-\, \langle \G , \conf(\widehat{\h}_\G,D) \rangle \,\,\leq\,\, 2L\E_X\big[\big\|\widehat{\boldeta}(X) \,-\, \boldeta(X)\big\|_1\big].
\end{equation*}
\end{lem*}

\begin{proof}

Let  $\h^*:\X\>\Delta_n$ be such that
\begin{equation*}
h^{*}_i(x) = 1  ~~\text{ if }~~ i \in \overline{\argmax}_{y \in [n]} \g_{y}^\top {\boldeta}(x)  ~.
\end{equation*}
Hence by Equation \ref{eqn:Bayes-decomposable} we have that
\begin{eqnarray*}
 \h^* &\in& \argmax_{\h:\X\> \Delta_n} \langle \G, \conf(\h,D) \rangle \;.
\end{eqnarray*}

We have that
\begin{eqnarray*}
\lefteqn{\max_{\h:\X\>\Delta_n} \langle \G , \conf(\h , D) \rangle \,-\, \langle \G , \conf(\widehat{\h}_\G,D) \rangle }\\[0.1em]
&=&
 \langle \G , \conf(\h^* , D) \rangle \,-\, \langle \G , \conf(\widehat{\h}_\G,D) \rangle  \\[0.1em]
&=&
\E_X  [{\boldeta}(X)]^\top [\G \h^*(X)]  - \E_X [{\boldeta}(X)]^\top [\G \widehat\h_\G(X)] \\[0.1em]
&=&
\E_X  [{\boldeta}(X)]^\top [\G \h^*(X)]  - \E_X [\boldeta(X) - \hat\boldeta(X)]^\top [\G \widehat\h_\G(X)] -  \E_X [\hat\boldeta(X)]^\top [\G \widehat\h_\G(X)]\\[0.1em]
&\leq &
\E_X  [{\boldeta}(X)]^\top [\G \h^*(X)]  - \E_X [\boldeta(X) - \hat\boldeta(X)]^\top [\G \widehat\h_\G(X)] -  \E_X [\hat\boldeta(X)]^\top [\G \h^*(X)]\\[0.1em]
&=&
\E_X [\boldeta(X) - \hat\boldeta(X)]^\top [\G] [\h^*(X)- \widehat\h_\G(X)] \\[0.1em]
&\leq&
2L\E_X ||{\boldeta}(X)-\hat{\boldeta}(X)||_1  \;. 
\end{eqnarray*}
\end{proof}

\subsection{Proof of Theorem \ref{thm:regret-bound-brute-force-plug-in}}
\begin{thm*}[\textbf{Regret Bound for Brute-force Plug-in Algorithm for Convex-like Non-decomposable Metrics}]
Let $D$ satisfy Assumption A, and $\psi: [0, 1]^{n\times n} \> \R_+$ satisfy Assumption B w.r.t. $D$. Furthermore, let $\psi$ be $L$-Lipschitz w.r.t. the $\ell_1$ norm over $\CC_D$, and be such that there exists $\xi > 0$ such that $\psi(\C) - \psi(\C') \,\leq\, \xi\langle \nabla\psi(\C), \C - \C' \rangle,\,\forall \C,\, \C' \in \CC_D$. If $\widehat{\h}_S$ is the classifier learned by Algorithm \ref{algo:brute-force-plug-in-non-decomposable} using training sample $S = (S', S'') \in (\X \times [n])^m$ with parameter $\alpha \in (0,1)$, then for any $\delta \in [0,1]$, we have with probability at least $1-\delta$ (over draw of $S$ from $D^m$):
\[
\reg_D^\psi[\widehat{\h}_{S}] \,\,\leq\,\, 2L\xi\E_X\big[\big\|\widehat{\boldeta}_{S'}(X) \,-\, \boldeta(X)\big\|_1\big] \,+\, 
2LC\sqrt{\frac{n^2\log(n)\log(\alpha m) + \log(n^2/\delta)}{\alpha m}},
\]
where $C > 0$ is a distribution-independent constant.
\end{thm*}
\begin{proof}
By Theorem \ref{thm:suff-cond-equivalence}, a $\psi$-optimal classifier exists. Let $\h^*: \X \> \Delta_n$ be one such classifier and let $\G^*=\nabla\psi(\conf(\h^*,D))$. Further, let $\h_{\G^*}: \X \> \Delta_n$ be a classifier such that $\big[\h_{\G^*}(x)]_i = 1$  if $i = \overline{\argmax}_{y \in [n]} \g_{y}^{*\top} {\boldeta}(x)$; then again by Theorem \ref{thm:suff-cond-equivalence}, $\Pf_D[\h_{\G^*}] = \Pf_D[\h^*]$. Also let $\widehat\h_{\G^*}: \X \> \Delta_n$ be  such that $\big[\widehat\h_{\G^*}(x)]_i = 1$  if $i = \overline{\argmax}_{y \in [n]} \g_{y}^{*\top} {\hat\boldeta_{S'}}(x)$. Thus,
\begin{eqnarray*}
{\reg_D^\psi[\widehat{\h}_S]}
&=& \Pf_D[\h^*] \,-\, \Pf_D[\widehat{\h}_S]\\
&=& \Pf_D[\h_{\G^*}] \,-\, \Pf_D[\widehat{\h}_S]\\
&=& \Pf_D[\h_{\G^*}] \,-\, \Pf_D[\widehat{\h}_{\G^*}]  \,\,+\,\, \Pf_D[\widehat{\h}_{\G^*}] \,-\, {\Pf}_{D_{S''}}[\widehat{\h}_{\G^*}] \,\,+\,\, {\Pf}_{D_{S''}}[\widehat{\h}_{\G^*}] \,-\, \Pf_D[\widehat{\h}_S]\\
&\leq& \Pf_D[\h_{\G^*}] \,-\, \Pf_D[\widehat{\h}_{\G^*}]  \,\,+\,\, \Pf_D[\widehat{\h}_{\G^*}] \,-\, {\Pf}_{D_{S''}}[\widehat{\h}_{\G^*}] \,\,+\,\, {\Pf}_{D_{S''}}[\widehat{\h}_{S}] \,-\, \Pf_D[\widehat{\h}_S]\\
&\leq& \Pf_D[\h_{\G^*}] \,-\, \Pf_D[\widehat{\h}_{\G^*}]  \,\,+\,\,
		\sup_{\h \,\in\, \H_{\widehat{\boldeta}_{S'}}}\big(\Pf_D[\h] \,-\, {\Pf}_{D_{S''}}[\h]\big) \,\,+\,\,
		\sup_{\h \,\in\, \H_{\widehat{\boldeta}_{S'}}}\big({\Pf}_{D_{S''}}[\h] \,-\, \Pf_D[\h]\big)\\
&=& \Pf_D[\h_{\G^*}] \,-\, \Pf_D[\widehat{\h}_{\G^*}]  \,\,+\,\,
		2 \sup_{\h \,\in\, \H_{\widehat{\boldeta}_{S'}}}\,\big|\Pf_D[\h] \,-\, {\Pf}_{D_{S''}}[\h]\big| \\
&=& \psi(\conf(\h_{\G^*},D)) -  \psi(\conf(\hat\h_{\G^*},D)) \,+\,
2 \sup_{\h \,\in\, \H_{\widehat{\boldeta}_{S'}}}\,\big|\psi(\conf(\h,D) \,-\, \psi(\conf(\h,D_{S''}))\big| \\
&\leq&
\xi \left( \langle \G^*, \conf(\h_{\G^*},D) \rangle -  \langle \G^*, \conf(\widehat\h_{\G^*},D) \rangle \right) \, + \,
2L \sup_{\h \,\in\, \H_{\widehat{\boldeta}_{S'}}}\,\big|\big|\conf(\h,D) \,-\, \conf(\h,D_{S''})\big|\big|_1 \\
&\leq&
2L\xi\E_X\big[\big\|\widehat{\boldeta}_{S'}(X) \,-\, \boldeta(X)\big\|_1\big] \,+\, 
2LC\sqrt{\frac{n^2\log(n)\log(\alpha m) + \log(n^2/\delta)}{\alpha m}},
\end{eqnarray*}
where the fourth step follows by definition of $\widehat{\h}_{S}$, the previous to last step follows from the `convexity-like' assumption on $\psi$,  and the last step follows from Lemmas \ref{lem:regret-bound-linear-psi-est-cond-prob} and \ref{lem:uconvg-conf}.
\end{proof}

\subsection{Proof of Theorem \ref{thm:regret-bound-fw-idealized}}
\begin{thm*}[\textbf{Regret Bound for (Idealized) Conditional Gradient Algorithm for Concave Smooth Non-decomposable Metrics}]
Let $\psi: [0, 1]^{n\times n} \> \R_+$ be concave over $\CC_D$ and $\beta$-smooth w.r.t. the $\ell_1$-norm  over $\CC_D$ . Let ${\h}^\FW$ be the classifier learned by Algorithm \ref{algo:FW-idealized} with parameters $\kappa \in \N$ and $\epsilon > 0$. Then
\[
\reg_D^\psi[{\h}^\FW] \,\,\leq\,\,2\epsilon \,+\, \frac{8\beta}{\kappa m+2}.
\]
\end{thm*}
\begin{proof}
 We use the result from \cite{Jaggi13}. To apply this result we must upper bound the `curvature constant' $C_\psi$ of $\psi$ and the approximation factor $\delta$ (which we call $\delta_\apx$).
 
 \begin{eqnarray*}
 C_\psi 
 &=&
 \sup_{\C_1,\C_2\in\CC_D, \gamma\in[0,1]} \frac{2}{\gamma^2} \Big( \psi\big(\C_1+\gamma(\C_2-\C_1)\big) -\psi\big(\C_1\big) - \gamma \big\langle \C_2-\C_1,\nabla\psi(\C_1) \big\rangle \Big)\\
 &\leq&
 \sup_{\C_1,\C_2\in\CC_D, \gamma\in[0,1]} \frac{2}{\gamma^2} \Big( \frac{\beta}{2} \gamma^2||\C_1-\C_2||^2_1 )\\
 &=&
 4\beta
\end{eqnarray*}
where the second step follows from the $\beta$-moothness of $\psi$ over $\CC_D$ w.r.t. the $\ell_1$ norm, and the last step follows from the observation that the entries of $\C_1$ and $\C_2$ sum to $1$ and are non-negative.

One can also see that the approximation factor $\delta_\apx \leq \frac{(T+1)\epsilon}{C_\psi}$.
Theorem 1 from \cite{Jaggi13} gives us
\begin{eqnarray*}
 \reg_D^\psi[{\h}^\FW] &=& \max_{\C\in\CC_D} \psi(\C) - \psi(\conf(\h^\FW)) \\
 &\leq& \frac{2C_\psi}{T+2}(1+\delta_\apx) \\
 &\leq& \frac{2C_\psi}{T+2}\left(1+\frac{(T+1)\epsilon}{C_\psi}\right) \\
 &\leq& \frac{8\beta}{T+2} + \frac{2(T+1)\epsilon}{T+2} \\
 &\leq& \frac{8\beta}{T+2} + 2\epsilon \\
 &=&\frac{8\beta}{\kappa m +2} + 2\epsilon.
\end{eqnarray*}
\end{proof}

\subsection{Proof of Lemma \ref{lem:approx-max-oracle}}

\begin{lem*}[\textbf{Approximation Factor for Linear Maximization Step in Algorithm \ref{algo:BFW}}]
Let $\psi: [0, 1]^{n\times n} \> \R_+$ satisfy the assumptions in Theorem \ref{thm:BFW-consistency}. Let $\widehat{\u}^j$ and $\widehat{\h}^{j}$ be the classifiers constructed in any given iteration $j$ of Algorithm \ref{algo:BFW} using training sample $S = (S', S'') \in (\X \times [n])^m$ and parameter $\alpha \in (0,1)$. Also, let $\G^j = \nabla \psi(\conf(\widehat{\h}^{j-1}, D))$. Then for any $\delta \in [0,1]$, we have with probability at least $1 - \delta$ (over draw of $S$ from $D^m$) for all $1\leq j\leq T$
\[
\langle {\G}^j , \conf(\widehat{\u}^j,D) \rangle \,\geq\, \underset{\u:\,\X \> \Delta_n}{\max}\,\langle {\G}^j , \conf(\u,D) \rangle \,-\, \widehat{\epsilon}_S,
\]
where
\[
\widehat{\epsilon}_S \,=\, 2L\E_X\big[\big\|\widehat{\boldeta}_{S'}(X) \,-\, \boldeta(X)\big\|_1\big] \,+\, 2C\beta n^2\sqrt{\frac{n^2\log(n)\log(\alpha m) + \log(n^2/\delta)}{\alpha m}},
\]
for a distribution-independent constant $C>0$.
\end{lem*}
\begin{proof}
 Let $1\leq j \leq T$.
 Let $\widehat \G^j= \nabla\psi(\conf(\widehat\h^{j-1},D_{S''}))$. Also let $\u^*\in\argmax_{\u:\X\>\Delta_n} \langle \G^j,\conf(\u,D) \rangle$. We then have by the definition of $\widehat\u^j$ and Lemma \ref{lem:regret-bound-linear-psi-est-cond-prob} that
 
 \begin{eqnarray}
\langle \widehat \G^j , \conf(\u^*,D) \rangle - \langle \widehat \G^j , \conf(\widehat\u^j,D) \rangle
&\leq&
\max_{\u:\X\>\Delta_n}  \langle \widehat \G^j , \conf(\u,D) \rangle - \langle \widehat \G^j , \conf(\widehat\u^j,D) \rangle \nonumber  \\
&\leq&
2L\E_X\big[\big\|\widehat{\boldeta}_{S'}(X) \,-\, \boldeta(X)\big\|_1\big] \label{eqn:lem-apx-factor-1}
 \end{eqnarray}

 Also
\begin{eqnarray}
\big|\big| \widehat \G^j - \G^j \big|\big|_\infty 
&=&
\big|\big| \nabla\psi(\conf(\widehat\h^{j-1},D_{S''})) -\nabla\psi(\conf(\widehat\h^{j-1},D)) \big|\big|_\infty \nonumber\\
&\leq&
\beta \big|\big| \conf(\widehat\h^{j-1},D_{S''}) -\conf(\widehat\h^{j-1},D) \big|\big|_1 \nonumber\\
&\leq&
\beta n^2 \big|\big| \conf(\widehat\h^{j-1},D_{S''}) -\conf(\widehat\h^{j-1},D) \big|\big|_\infty \nonumber\\
&\leq&
\beta n^2 \max_{k\in[j-1]} \big|\big| \conf(\widehat\u^{k},D_{S''}) -\conf(\widehat\u^{k},D) \big|\big|_\infty \nonumber\\
&\leq&
\beta n^2 \sup_{\h\in\H_{\hat\boldeta_{S'}}} \big|\big| \conf(\h,D_{S''}) -\conf(\h,D) \big|\big|_\infty 
\label{eqn:lem-apx-factor-2}
\end{eqnarray}

We then have
\begin{eqnarray}
\lefteqn{\underset{\u:\,\X \> \Delta_n}{\max}\,\langle {\G}^j , \conf(\u,D) \rangle - \langle {\G}^j , \conf(\widehat{\u}^j,D) \rangle} \nonumber \\
&=&
\langle {\G}^j , \conf(\u^*,D) \rangle - \langle {\G}^j , \conf(\widehat{\u}^j,D) \rangle \nonumber \\
&=&
  \langle {\G}^j , \conf(\u^*,D) \rangle   -  \langle {\widehat\G}^j , \conf(\u^*,D) \rangle
+ \langle {\widehat\G}^j , \conf(\u^*,D) \rangle   -  \langle {\G}^j , \conf(\widehat{\u}^j,D) \rangle \nonumber \\
&\leq&
\big|\big| \G^j - \widehat\G^j \big|\big|_\infty  \big|\big|\conf(\u^*,D)\big|\big|_1
+ \langle {\widehat\G}^j , \conf(\u^*,D) \rangle   -  \langle {\G}^j , \conf(\widehat{\u}^j,D) \rangle \nonumber \\
&=& \big|\big| \G^j - \widehat\G^j \big|\big|_\infty 
+ \langle {\widehat\G}^j , \conf(\u^*,D) \rangle   -  \langle {\G}^j , \conf(\widehat{\u}^j,D) \rangle \nonumber \\
&=& \big|\big| \G^j - \widehat\G^j \big|\big|_\infty 
+ \langle {\widehat\G}^j , \conf(\u^*,D) \rangle  - \langle \widehat \G^j , \conf(\widehat\u^j,D) \rangle
+ \langle \widehat \G^j , \conf(\widehat\u^j,D) \rangle  -  \langle {\G}^j , \conf(\widehat{\u}^j,D) \rangle \nonumber \\
&\leq&
\big|\big| \G^j - \widehat\G^j \big|\big|_\infty 
+ 2L\E_X\big[\big\|\widehat{\boldeta}_{S'}(X) \,-\, \boldeta(X)\big\|_1\big]
+ \langle \widehat \G^j , \conf(\widehat\u^j,D) \rangle  -  \langle {\G}^j , \conf(\widehat{\u}^j,D) \rangle \nonumber \\
&\leq&
\big|\big| \G^j - \widehat\G^j \big|\big|_\infty 
+ 2L\E_X\big[\big\|\widehat{\boldeta}_{S'}(X) \,-\, \boldeta(X)\big\|_1\big]
+ \big|\big| \G^j - \widehat\G^j \big|\big|_\infty  \big|\big|\conf(\widehat \u^j,D)\big|\big|_1 \nonumber \\
&=&
2\big|\big| \G^j - \widehat\G^j \big|\big|_\infty 
+ 2L\E_X\big[\big\|\widehat{\boldeta}_{S'}(X) \,-\, \boldeta(X)\big\|_1\big]  \nonumber \\
&\leq&
2\beta n^2 \sup_{\h\in\H_{\hat\boldeta_{S'}}} \big|\big| \conf(\h,D_{S''}) -\conf(\h,D) \big|\big|_\infty 
+ 2L\E_X\big[\big\|\widehat{\boldeta}_{S'}(X) \,-\, \boldeta(X)\big\|_1\big]  \label{eqn:lem-apx-factor-3} 
\end{eqnarray}
where the first and third inequalities in the above are due to the Holder's inequality, the second inequality is due to Equation \ref{eqn:lem-apx-factor-1} and the last inequality is due to Equation \ref{eqn:lem-apx-factor-2}.

Applying Lemma \ref{lem:uconvg-conf} to Equation \ref{eqn:lem-apx-factor-3} the proof is complete.
\end{proof}

\subsection{Proof of Theorem \ref{thm:BFW-consistency-non-smooth}}
\begin{thm*}[\textbf{Regret Bound for Sample-based Conditional Gradient Algorithm for a Larger Family of Non-decomposable Metrics}]
Let $S = (S', S'') \in (\X \times [n])^m$ be the given training sample drawn i.i.d. from distribution $D$. Let $\psi: [0, 1]^{n\times n} \> \R_+$ be such that for any $\rho \in (0,1)$, there exists $\psi_\rho: [0, 1]^{n\times n} \> \R_+$ which is concave over $\CC_D$, $L_\rho$-Lipschitz w.r.t. the $\ell_1$ norm over $\CC_{D_{S''}}$ and $\beta_\rho$-smooth w.r.t. the $\ell_1$-norm, with 
$$\sup_{\C\,\in\,\CC_D} |\psi(\C) - \psi_\rho(\C)| \leq \theta(\rho),$$ for some strictly increasing function $\theta: \R_+ \> \R_+$. Let $\widehat{\h}^{\FW,\rho}_S$ be the classifier learned by Algorithm \ref{algo:BFW} when applied to $\psi_\rho$ with training sample $S$ and parameters $\kappa \in \N$ and $\alpha \in (0,1)$. Then for any $\delta \in [0,1]$, we have with probability at least $1-\delta$ (over draw of $S$ from $D^m$)
\begin{eqnarray*}
\reg_D^\psi[\widehat{\h}^{\FW,\rho}_{S}] &\leq& 4L_\rho\E_X\big[\big\|\widehat{\boldeta}_{S'}(X) \,-\, \boldeta(X)\big\|_1\big] \,+\,  4\beta_\rho n^2C\sqrt{\frac{n^2\log(n)\log(\alpha m) + \log(n^2/\delta)}{\alpha m}}\\ 
&& \hspace{9.5cm}
+\, \frac{8\beta_\rho}{\kappa m+2} \,+\, 2\theta(\rho),
\end{eqnarray*}
where $C > 0$ is a distribution-independent constant. 
\end{thm*}
\begin{proof}
 From Theorem \ref{thm:BFW-consistency} we have that
 \[
\reg_D^{\psi_\rho}[\widehat{\h}^{\FW,\rho}_{S}] \,\,\leq\,\, 4L_\rho\E_X\big[\big\|\widehat{\boldeta}_{S'}(X) \,-\, \boldeta(X)\big\|_1\big] \,+\, 4\beta_\rho n^2C\sqrt{\frac{n^2\log(n)\log(\alpha m) + \log(n^2/\delta)}{\alpha m}} \,+\, \frac{8\beta_\rho}{\kappa m+2}\;.
\]
 
 For simplicity assume that the maximizer of $\psi(\conf(\h,D))$ over $\h:\X\>\Delta_n$ exists. Let $\h^*\in \argmax_{\h:\X\>\Delta_n} \psi(\conf(\h,D))$. We then have that
\begin{eqnarray*}
  \reg_D^{\psi}[\widehat{\h}^{\FW,\rho}_{S}] 
  &=& 
  \sup_{\h:\X\>\Delta_n} \psi(\conf(\h,D)) - \psi(\conf(\widehat{\h}^{\FW,\rho}_{S},D)) \\
  &=&
  \psi(\conf(\h^*,D)) - \psi(\conf(\widehat{\h}^{\FW,\rho}_{S},D)) \\
  &\leq&
  \psi_\rho(\conf(\h^*,D)) - \psi_\rho(\conf(\widehat{\h}^{\FW,\rho}_{S},D)) + 2\theta(\rho) \\
  &\leq&
  \max_{\h:\X\>\Delta_n} \psi_\rho(\conf(\h,D)) - \psi_\rho(\conf(\widehat{\h}^{\FW,\rho}_{S},D)) + 2\theta(\rho) \\
  &=&
  \reg_D^{\psi_\rho}[\widehat{\h}^{\FW,\rho}_{S}] +  2\theta(\rho)
\end{eqnarray*}
\end{proof}

\section{Details of Calculations for Smoothed Performance Metrics in Table \ref{tab:perf-metrics-smoothed}}
\label{sec:smoothed-metrics}
We now give details of derivation of the function $\theta$ (defined in \Thm{thm:BFW-consistency-non-smooth}), Lipschitz constant $L_\rho$, and the smoothness parameter $\beta_\rho$ for the smoothed performance metric $\psi_\rho$. In each case, we make use the fact that the Lipschitz constant can be obtained by bounding the maximum absolute entry ($\ell_\infty$ norm) of the gradient of $\psi$ and the smoothness parameter is obtained by bounding the maximum absolute entry ($\ell_\infty$ norm) of its Hessian.
 
\noindent \textbf{H-mean.} For the H-mean, $\psi^\text{H}(\C) = n\Big({\sum_{y=1}^n \frac{\sum_{\widehat{y} = 1}^n C_{y,\widehat{y}}}{C_{y,y}}}\Big)^{-1}$ is $\frac{n}{\pi_{\min}}$-Lipschitz over $\CC_D$. Hence we have $\theta(\rho) = \frac{n}{\pi_{\min}}\rho$. The gradient of $\psi^\text{H}_\rho$ is given by:
\begin{eqnarray*} 
\nabla_{C_{uu'}} \psi^\text{H}_\rho (\C) \,=\, 
\begin{cases}
\frac{n\frac{\sum_{\hy \ne u} C_{u,\hy}}{(C_{u,u}+\rho)^2}}{\Big(\sum_{y=1}^n \frac{\sum_{\hy=1}^n C_{y,\hy} \,+\, \rho}{C_{y,y}+\rho}\Big)^2} 
&\text{if $u = u'$}\\
-\frac{n\frac{1}{C_{u,u}+\rho}}{\Big(\sum_{y=1}^n \frac{\sum_{\hy=1}^n C_{y,\hy} \,+\, \rho}{C_{y,y}+\rho}\Big)^2} &\text{otherwise}
\end{cases}.
\end{eqnarray*}
\noindent The Lipschitz constant $L_\rho$ for $\psi^\text{H}_\rho$ is then given by a bound on the $\ell_\infty$ norm of the above gradient. 
\begin{eqnarray*}
\|\nabla \psi^\text{H}_\rho (\C)\|_\infty &\leq& \max_{u \,\in\, [n]}\, \frac{n\frac{\sum_{\hy = 1}^n C_{u,\hy} + \rho}{(C_{u,u}+\rho)^2}}{\Big(\sum_{y=1}^n \frac{\sum_{\hy=1}^n C_{y,\hy} \,+\, \rho}{C_{y,y}+\rho}\Big)^2}\\
&\leq & \max_{u \,\in\, [n]}\, \frac{\frac{n}{\sum_{\hy = 1}^n C_{u,\hy} + \rho}\Big(\frac{\sum_{\hy = 1}^n C_{u,\hy} + \rho}{C_{u,u}+\rho}\Big)^2}{\Big(\sum_{y=1}^n \frac{\sum_{\hy=1}^n C_{y,\hy} \,+\, \rho}{C_{y,y}+\rho}\Big)^2}\\
&=& \max_{u \,\in\, [n]}\, \frac{n}{\pi_u + \rho}\Bigg(\frac{\frac{\sum_{\hy = 1}^n C_{u,\hy} + \rho}{C_{u,u}+\rho}}{\sum_{y=1}^n \frac{\sum_{\hy=1}^n C_{y,\hy} \,+\, \rho}{C_{y,y}+\rho}}\Bigg)^2 \\
&\leq& \max_{u \,\in\, [n]}\, \frac{n}{\pi_u + \rho}\\
& \leq & \frac{n}{\rho}.
\end{eqnarray*}
\noindent Next, we calculate the smoothness parameter $\beta_\rho$ of $\psi^{\text{H}}_\rho$ by computing the Hessian of $\psi^{\text{H}}_\rho$:
\begin{eqnarray*}
\lefteqn{\nabla^2_{C_{uu'}, C_{vv'}} \psi^\text{H}_\rho(\C) \,=\,}\\
&\hspace{2cm} 
\begin{cases}
\frac{-2n\sum_{\hy \ne u} C_{u,\hy}}{(C_{u,u}+\rho)^2}\,\frac{\Big(\sum_{y=1}^n \frac{\sum_{\hy=1}^n C_{y,\hy} \,+\, \rho}{C_{y,y}+\rho}\Big)\frac{1}{C_{u,u}+\rho} \,-\, 1}{\Big(\sum_{y=1}^n \frac{\sum_{\hy=1}^n C_{y,\hy} \,+\, \rho}{C_{y,y}+\rho}\Big)^3}
&\text{if $u = u' = v = v'$}\\
\frac{-2n}{C_{u,u}+\rho}\,\frac{\Big(\sum_{y=1}^n \frac{\sum_{\hy=1}^n C_{y,\hy} \,+\, \rho}{C_{y,y}+\rho}\Big)\frac{\sum_{\hy \ne u} C_{u,\hy}}{(C_{u,u}+\rho)^2} \,+\, 1}{\Big(\sum_{y=1}^n \frac{\sum_{\hy=1}^n C_{y,\hy} \,+\, \rho}{C_{y,y}+\rho}\Big)^3}
&\text{if $u = u' = v \ne v'$}\\
\frac{2n}{(C_{u,u}+\rho)^2}\,\frac{\Big(\sum_{y=1}^n \frac{\sum_{\hy=1}^n C_{y,\hy} \,+\, \rho}{C_{y,y}+\rho}\Big) \,-\, \sum_{\hy \ne u} C_{u,\hy}}{\Big(\sum_{y=1}^n \frac{\sum_{\hy=1}^n C_{y,\hy} \,+\, \rho}{C_{y,y}+\rho}\Big)^3}
&\text{if $u \ne u' = v = v'$}\\
\frac{-2n}{C_{u,u}+\rho}\,\frac{\Big(\sum_{y=1}^n \frac{\sum_{\hy=1}^n C_{y,\hy} \,+\, \rho}{C_{y,y}+\rho}\Big)\frac{1}{C_{u,u}+\rho} \,-\, 1}{\Big(\sum_{y=1}^n \frac{\sum_{\hy=1}^n C_{y,\hy} \,+\, \rho}{C_{y,y}+\rho}\Big)^3}
&\text{if $u \ne u' = v \ne v'$}\\
\frac{-2n\frac{\sum_{\hy \ne u} C_{u,\hy}}{(C_{u,u}+\rho)^2}\frac{\sum_{\hy \ne v} C_{v,\hy}}{(C_{v,v}+\rho)^2}}{\Big(\sum_{y=1}^n \frac{\sum_{\hy=1}^n C_{y,\hy} \,+\, \rho}{C_{y,y}+\rho}\Big)^3} 
&\text{if $u = u' \ne v = v'$}
\\
\frac{2n\frac{\sum_{\hy \ne u} C_{u,\hy}}{(C_{u,u}+\rho)^2}\frac{1}{C_{v,v}+\rho}}{\Big(\sum_{y=1}^n \frac{\sum_{\hy=1}^n C_{y,\hy} \,+\, \rho}{C_{y,y}+\rho}\Big)^3} 
&\text{if $u = u' \ne v \ne v'$}\\
\frac{2n\frac{1}{C_{u,u}+\rho}\frac{\sum_{\hy \ne v} C_{v,\hy}}{(C_{v,v}+\rho)^2}}{\Big(\sum_{y=1}^n \frac{\sum_{\hy=1}^n C_{y,\hy} \,+\, \rho}{C_{y,y}+\rho}\Big)^3} 
&\text{if $u \ne u' \ne v = v'$}\\
\frac{-2n\frac{1}{C_{u,u}+\rho}\frac{1}{C_{v,v}+\rho}}{\Big(\sum_{y=1}^n \frac{\sum_{\hy=1}^n C_{y,\hy} \,+\, \rho}{C_{y,y}+\rho}\Big)^3} 
&\text{otherwise}.
\end{cases}
\end{eqnarray*}

\noindent Bounding the entry of the Hessian matrix corresponding to confusion matrix entries $C_{u,u'}$ and $C_{v,v'}$ for $u' = u = v = v'$, we get
\begin{eqnarray*}
\lefteqn{|\nabla^2_{C_{uu'},C_{vv'}} \psi^\text{Q}_\rho(\C)|}\\
&\leq &
\frac{2n\sum_{\hy \ne u} C_{u,\hy}}{(C_{u,u}+\rho)^2}\,\frac{\Big(\sum_{y=1}^n \frac{\sum_{\hy=1}^n C_{y,\hy} \,+\, \rho}{C_{y,y}+\rho}\Big)\frac{1}{C_{u,u}+\rho} \,+\, 1}{\Big(\sum_{y=1}^n \frac{\sum_{\hy=1}^n C_{y,\hy} \,+\, \rho}{C_{y,y}+\rho}\Big)^3}\\
&\leq &
2n
\Bigg[
\frac{\frac{\sum_{\hy = 1}^n C_{u,\hy} \,+\, \rho}{C_{u,u}+\rho}}{\sum_{y=1}^n \frac{\sum_{\hy=1}^n C_{y,\hy} \,+\, \rho}{C_{y,y}+\rho}}
\Bigg]^3\,
\frac{C_{u,u}+\rho}{\big(\sum_{\hy = 1}^n C_{u,\hy} \,+\, \rho\big)^2}\,
\Bigg[
\Bigg(\sum_{y=1}^n \frac{\sum_{\hy=1}^n C_{y,\hy} \,+\, \rho}{C_{y,y}+\rho}\Bigg)\frac{1}{C_{u,u}+\rho} \,+\, 1
\Bigg]\\
&\leq &
2n
\Bigg[
\Bigg(\sum_{y=1}^n \frac{\sum_{\hy=1}^n C_{y,\hy} \,+\, \rho}{C_{y,y}+\rho}\Bigg)\frac{1}{\big(\sum_{\hy = 1}^n C_{u,\hy} \,+\, \rho\big)^2} \,+\, \frac{C_{u,u}+\rho}{\big(\sum_{\hy = 1}^n C_{u,\hy} \,+\, \rho\big)^2}\,
\Bigg]\\
&\leq &
2n
\bigg[
n\frac{1+\rho}{\rho^3} \,+\, \frac{1+\rho}{\rho^2}\,
\bigg]\\
& \leq &
4n^2\bigg(\frac{1+\rho}{\rho^3}\bigg).
\end{eqnarray*}
The above bound can be shown to hold for all entries for which $u = v$. We next consider the case when $u' \ne u \ne v \ne v'$; assuming w.l.o.g. that $C_{u,u} < C_{v,v}$, we have
\begin{eqnarray*}
{|\nabla^2_{C_{uu'},C_{vv'}} \psi^\text{Q}_\rho(\C)|}
&\leq &
2n\Bigg[\frac{\frac{\sum_{\hy=1}^n C_{u,\hy} \,+\, \rho}{C_{u,u}+\rho}}{\sum_{y=1}^n \frac{\sum_{\hy=1}^n C_{y,\hy} \,+\, \rho}{C_{y,y}+\rho}}\Bigg]^3\frac{1}{(\sum_{\hy=1}^n C_{u,\hy} \,+\, \rho)^3}\\
&\leq & \frac{2n}{\rho^3},
\end{eqnarray*}
where the same bound holds for all Hessian entries corresponding to $u \ne v$. The smoothness parameter in Table \ref{tab:perf-metrics-smoothed} then follows from the above bounds.


\noindent \textbf{Q-mean.} For the Q-mean, $\psi^\text{Q}(\C) = 1-\sqrt{\frac{1}{n}\sum_{y=1}^n\Big(1-\frac{C_{y,y}}{\sum_{\widehat{y} = 1}^n C_{y,\widehat{y}}}\Big)^2}$ is $\frac{1}{\pi_{\min}\sqrt{n}}$-Lipschitz over $\CC_D$. Hence we have $\theta(\rho) = \frac{1}{\sqrt{n}\pi_{\min}}\rho$. The gradient and Hessian for $\psi^\text{Q}_\rho$ are given by:

\[ 
\nabla_{C_{uu'}} \psi^\text{Q}_\rho (\C) \,=\, 
\begin{cases}
\frac{1}{\sqrt{n}}
\frac{\frac{(\sum_{\widehat{y} \ne u} C_{u,\widehat{y}}\,+\,\rho)^2}{(\sum_{\widehat{y} = 1}^n C_{u,\widehat{y}}\,+\,\rho)^3}}
{\sqrt{\sum_{y=1}^n\Big(\frac{\sum_{\widehat{y} \ne y} C_{y,\widehat{y}}\,+\,\rho}{\sum_{\widehat{y} = 1}^n C_{y,\widehat{y}}\,+\,\rho}\Big)^2}}
&\text{if $u = u'$}\\
-\frac{1}{\sqrt{n}}
\frac{\frac{(\sum_{\widehat{y} \ne u} C_{u,\widehat{y}}\,+\,\rho)C_{u,u}}{(\sum_{\widehat{y} = 1}^n C_{u,\widehat{y}}\,+\,\rho)^3}}
{\sqrt{\sum_{y=1}^n\Big(\frac{\sum_{\widehat{y} \ne y} C_{y,\widehat{y}}\,+\,\rho}{\sum_{\widehat{y} = 1}^n C_{y,\widehat{y}}\,+\,\rho}\Big)^2}}
 &\text{otherwise}
\end{cases}.
\] 
\noindent We next calculate the Lipschitz constant $L_\rho$ for $\psi^\text{Q}_\rho$ by bounding the $\ell_\infty$ norm of its gradient.
\begin{eqnarray*}
\|\nabla \psi^\text{Q}_\rho (\C)\|_\infty 
&\leq& 
\max_{u \,\in\, [n]}\, \frac{1}{\sqrt{n}}
\frac{\frac{1}{\sum_{\widehat{y} = 1}^n C_{u,\widehat{y}}\,+\,\rho}}
{\sqrt{\sum_{y=1}^n\Big(\frac{\sum_{\widehat{y} \ne y} C_{y,\widehat{y}}\,+\,\rho}{\sum_{\widehat{y} = 1}^n C_{y,\widehat{y}}\,+\,\rho}\Big)^2}}\\
&=&
\max_{u \,\in\, [n]}\, \frac{1}{\sqrt{n}}
\frac{1}{\sum_{\widehat{y} \ne y} C_{y,\widehat{y}}\,+\,\rho} \frac{\frac{\sum_{\widehat{y} \ne y} C_{y,\widehat{y}}\,+\,\rho}{\sum_{\widehat{y} = 1}^n C_{u,\widehat{y}}\,+\,\rho}}{\sqrt{\sum_{y=1}^n\Big(\frac{\sum_{\widehat{y} \ne y} C_{y,\widehat{y}}\,+\,\rho}{\sum_{\widehat{y} = 1}^n C_{y,\widehat{y}}\,+\,\rho}\Big)^2}}\\
&\leq& \max_{u \,\in\, [n]}\, \frac{1}{\sqrt{n}}
\frac{1}{\sum_{\widehat{y} \ne y} C_{y,\widehat{y}}\,+\,\rho}\\
& \leq & \frac{1}{\sqrt{n}\rho}.
\end{eqnarray*}

\noindent We next calculate the Hessian and bound its norm.
\begin{eqnarray*}
\lefteqn{\nabla^2_{C_{uu'},C_{vv'}} \psi^\text{Q}_\rho(\C) \,\,=\,}\\
&
\begin{cases}
-\frac{1}{\sqrt{n}}
\frac{\frac{
(\sum_{\widehat{y} \ne y} C_{u,\widehat{y}}\,+\,\rho)^2}{(\sum_{\widehat{y} = 1}^n C_{u,\widehat{y}}\,+\,\rho)^3}
\Big(3{\sum_{y=1}^n\Big(\frac{\sum_{\widehat{y} \ne y} C_{y,\widehat{y}}\,+\,\rho}{\sum_{\widehat{y} = 1}^n C_{y,\widehat{y}}\,+\,\rho}\Big)^2 
\frac{1}{\sum_{\widehat{y} = 1}^n C_{u,\widehat{y}}\,+\,\rho}
\,+\, 1
\Big)}
}
{\Big({\sum_{y=1}^n\Big(\frac{\sum_{\widehat{y} \ne y} C_{y,\widehat{y}}\,+\,\rho}{\sum_{\widehat{y} = 1}^n C_{y,\widehat{y}}\,+\,\rho}\Big)^2\Big)^{3/2}}}
& \text{if $u' = u = v = v'$}\\
\frac{1}{\sqrt{n}}
\frac{\frac{
\sum_{\widehat{y} \ne y} C_{u,\widehat{y}}\,+\,\rho}{(\sum_{\widehat{y} = 1}^n C_{u,\widehat{y}}\,+\,\rho)^3}
\Big({\sum_{y=1}^n\Big(\frac{\sum_{\widehat{y} \ne y} C_{y,\widehat{y}}\,+\,\rho}{\sum_{\widehat{y} = 1}^n C_{y,\widehat{y}}\,+\,\rho}\Big)^2 
\frac{-\sum_{\widehat{y} \ne y} C_{y,\widehat{y}}\,-\,\rho \,+\, 2C_{u,u}}{\sum_{\widehat{y} = 1}^n C_{u,\widehat{y}}\,+\,\rho}
\,+\, C_{u,u}
\Big)}
}
{\Big({\sum_{y=1}^n\Big(\frac{\sum_{\widehat{y} \ne y} C_{y,\widehat{y}}\,+\,\rho}{\sum_{\widehat{y} = 1}^n C_{y,\widehat{y}}\,+\,\rho}\Big)^2\Big)^{3/2}}}
& \text{if $u' = u = v \ne v'$}\\
-\frac{1}{\sqrt{n}}
\frac{\frac{
\sum_{\widehat{y} \ne u} C_{u,\widehat{y}}\,+\,\rho}{(\sum_{\widehat{y} = 1}^n C_{u,\widehat{y}}\,+\,\rho)^3}
\Big({\sum_{y=1}^n\Big(\frac{\sum_{\widehat{y} \ne y} C_{y,\widehat{y}}\,+\,\rho}{\sum_{\widehat{y} = 1}^n C_{y,\widehat{y}}\,+\,\rho}\Big)^2 
\frac{\sum_{\widehat{y} \ne y} C_{y,\widehat{y}}\,+\,\rho \,-\, 2C_{u,u}}{\sum_{\widehat{y} = 1}^n C_{u,\widehat{y}}\,+\,\rho}
\,-\, \sum_{\widehat{y} \ne u} C_{u,\widehat{y}} \,-\, \rho
\Big)}
}
{\Big({\sum_{y=1}^n\Big(\frac{\sum_{\widehat{y} \ne y} C_{y,\widehat{y}}\,+\,\rho}{\sum_{\widehat{y} = 1}^n C_{y,\widehat{y}}\,+\,\rho}\Big)^2\Big)^{3/2}}}
& \text{if $u' \ne u = v = v'$}\\
-\frac{1}{\sqrt{n}}
\frac{\frac{
C_{u,u}}{(\sum_{\widehat{y} = 1}^n C_{u,\widehat{y}}\,+\,\rho)^3}
\Big({\sum_{y=1}^n\Big(\frac{\sum_{\widehat{y} \ne y} C_{y,\widehat{y}}\,+\,\rho}{\sum_{\widehat{y} = 1}^n C_{y,\widehat{y}}\,+\,\rho}\Big)^2 
\frac{-2(\sum_{\widehat{y} \ne u} C_{u,\widehat{y}}\,+\,\rho) \,+\, C_{u,u}}{\sum_{\widehat{y} = 1}^n C_{u,\widehat{y}}\,+\,\rho}
\,+\, \sum_{\widehat{y} \ne u} C_{u,\widehat{y}}\,+\,\rho
\Big)}
}{\Big({\sum_{y=1}^n\Big(\frac{\sum_{\widehat{y} \ne y} C_{y,\widehat{y}}\,+\,\rho}{\sum_{\widehat{y} = 1}^n C_{y,\widehat{y}}\,+\,\rho}\Big)^2\Big)^{3/2}}}
& \text{if $u' \ne u = v \ne v'$}\\
\frac{1}{\sqrt{n}}
\frac{\frac{(\sum_{\widehat{y} \ne u} C_{u,\widehat{y}}\,+\,\rho)^2}{(\sum_{\widehat{y} = 1}^n C_{u,\widehat{y}}\,+\,\rho)^3}
\frac{(\sum_{\widehat{y} \ne y} C_{v,\widehat{y}}\,+\,\rho)^2}{(\sum_{\widehat{y} = 1}^n C_{v,\widehat{y}}\,+\,\rho)^3}}
{\Big({\sum_{y=1}^n\Big(\frac{\sum_{\widehat{y} \ne y} C_{y,\widehat{y}}\,+\,\rho}{\sum_{\widehat{y} = 1}^n C_{y,\widehat{y}}\,+\,\rho}\Big)^2\Big)^{3/2}}}
& \text{if $u' = u \ne v = v'$}
\\
-\frac{1}{\sqrt{n}}
\frac{\frac{(\sum_{\widehat{y} \ne u} C_{u,\widehat{y}}\,+\,\rho)C_{u,u}}{(\sum_{\widehat{y} = 1}^n C_{u,\widehat{y}}\,+\,\rho)^3}
\frac{(\sum_{\widehat{y} \ne y} C_{v,\widehat{y}}\,+\,\rho)^2}{(\sum_{\widehat{y} = 1}^n C_{v,\widehat{y}}\,+\,\rho)^3}}
{\Big({\sum_{y=1}^n\Big(\frac{\sum_{\widehat{y} \ne y} C_{y,\widehat{y}}\,+\,\rho}{\sum_{\widehat{y} = 1}^n C_{y,\widehat{y}}\,+\,\rho}\Big)^2\Big)^{3/2}}}
& \text{if $u' \ne u \ne v = v'$}\\
-\frac{1}{\sqrt{n}}
\frac{\frac{(\sum_{\widehat{y} \ne u} C_{u,\widehat{y}}\,+\,\rho)^2}{(\sum_{\widehat{y} = 1}^n C_{u,\widehat{y}}\,+\,\rho)^3}
\frac{(\sum_{\widehat{y} \ne y} C_{v,\widehat{y}}\,+\,\rho)C_{v,v}}{(\sum_{\widehat{y} = 1}^n C_{v,\widehat{y}}\,+\,\rho)^3}}
{\Big({\sum_{y=1}^n\Big(\frac{\sum_{\widehat{y} \ne y} C_{y,\widehat{y}}\,+\,\rho}{\sum_{\widehat{y} = 1}^n C_{y,\widehat{y}}\,+\,\rho}\Big)^2\Big)^{3/2}}}
& \text{if $u' = u \ne v \ne v'$}\\
\frac{1}{\sqrt{n}}
\frac{\frac{(\sum_{\widehat{y} \ne u} C_{u,\widehat{y}}\,+\,\rho)C_{u,u}}{(\sum_{\widehat{y} = 1}^n C_{u,\widehat{y}}\,+\,\rho)^3}
\frac{(\sum_{\widehat{y} \ne y} C_{v,\widehat{y}}\,+\,\rho)C_{v,v}}{(\sum_{\widehat{y} = 1}^n C_{v,\widehat{y}}\,+\,\rho)^3}}
{\Big({\sum_{y=1}^n\Big(\frac{\sum_{\widehat{y} \ne y} C_{y,\widehat{y}}\,+\,\rho}{\sum_{\widehat{y} = 1}^n C_{y,\widehat{y}}\,+\,\rho}\Big)^2\Big)^{3/2}}}
& \text{otherwise}.
\end{cases}
\end{eqnarray*}

We now obtain the smoothness parameter $\beta_\rho$ for $\psi^\text{Q}_\rho$. We start by bounding the entry of the Hessian matrix corresponding to confusion matrix entries $C_{u,u'}$ and $C_{v,v'}$ where $u' = u = v = v'$; the same bound can be shown to hold for all entries for which $u = v$.
\begin{eqnarray*}
\lefteqn{|\nabla^2_{C_{uu'},C_{vv'}} \psi^\text{Q}_\rho(\C)|}\\
&= &
\frac{1}{\sqrt{n}}
\Bigg[
\frac{
\Big(
\frac{\sum_{\widehat{y} \ne u} C_{u,\widehat{y}}\,+\,\rho}{\sum_{\widehat{y} = 1}^n C_{u,\widehat{y}}\,+\,\rho}\Big)^2}
{{\sum_{y=1}^n\Big(\frac{\sum_{\widehat{y} \ne y} C_{y,\widehat{y}}\,+\,\rho}{\sum_{\widehat{y} = 1}^n C_{y,\widehat{y}}\,+\,\rho}\Big)^2}}
\Bigg]^{3/2}
\frac{1}{\sum_{\widehat{y} = 1}^n C_{u,\widehat{y}}\,+\,\rho}
\Bigg[3{\sum_{y=1}^n\Bigg(\frac{\sum_{\widehat{y} \ne y} C_{y,\widehat{y}}\,+\,\rho}{\sum_{\widehat{y} = 1}^n C_{y,\widehat{y}}\,+\,\rho}\Bigg)^2 
\frac{1}{\sum_{\widehat{y} = 1}^n C_{u,\widehat{y}}\,+\,\rho}
\,+\, 1
\Bigg]}\\
&\leq &
\frac{1}{\sqrt{n}}\,
\frac{1}{\sum_{\widehat{y} = 1}^n C_{u,\widehat{y}}\,+\,\rho}\,
\Bigg[3{\sum_{y=1}^n\Bigg(\frac{\sum_{\widehat{y} \ne y} C_{y,\widehat{y}}\,+\,\rho}{\sum_{\widehat{y} = 1}^n C_{y,\widehat{y}}\,+\,\rho}\Bigg)^2 
\frac{1}{\sum_{\widehat{y} = 1}^n C_{u,\widehat{y}}\,+\,\rho}
\,+\, 1
\Bigg]}\\
&\leq &
\frac{1}{\sqrt{n}}\,
\frac{1}{\sum_{\widehat{y} = 1}^n C_{u,\widehat{y}}\,+\,\rho}\,
\Bigg[\frac{3n}{\sum_{\widehat{y} = 1}^n C_{u,\widehat{y}}\,+\,\rho}
\,+\, 1
\Bigg]\\
&\leq &
\frac{1}{\sqrt{n}}\,
\frac{1}{\rho}\,
\Bigg[\frac{3n}{\rho}
\,+\, 1
\Bigg]\\
& \leq &
\frac{4\sqrt{n}}{\rho^2}.
\end{eqnarray*}

\noindent We next bound the entry of the Hessian matrix corresponding to $u' \ne u \ne v \ne v'$; the same bound can be shown to hold for all entries where $u \ne v$. Assuming w.l.o.g. that $\sum_{\widehat{y} \ne u} C_{u,\widehat{y}} \,>\, \sum_{\widehat{y} \ne v} C_{v,\widehat{y}}$,
\begin{eqnarray*}
|\nabla^2_{C_{uu'},C_{vv'}} \psi^\text{Q}_\rho(\C)|
&\leq&
\frac{1}{\sqrt{n}}
\Bigg[
\frac{
\Big(\frac{\sum_{\widehat{y} \ne u} C_{u,\widehat{y}}\,+\,\rho}{\sum_{\widehat{y} = 1}^n C_{u,\widehat{y}}\,+\,\rho}\Big)^2}
{{\sum_{y=1}^n\Big(\frac{\sum_{\widehat{y} \ne y} C_{y,\widehat{y}}\,+\,\rho}{\sum_{\widehat{y} = 1}^n C_{y,\widehat{y}}\,+\,\rho}\Big)^2}}
\Bigg]^{3/2}
\frac{C_{u,u}}{\sum_{\widehat{y} \ne u} C_{u,\widehat{y}}\,+\,\rho}\,\,\frac{C_{v,v}}{(\sum_{\widehat{y} = 1}^n C_{v,\widehat{y}}\,+\,\rho)^3}\\
&\leq &
\frac{1}{\sqrt{n}}\,\,
\frac{C_{u,u}}{\sum_{\widehat{y} \ne u} C_{u,\widehat{y}}\,+\,\rho}\,\,\frac{C_{v,v}}{(\sum_{\widehat{y} = 1}^n C_{v,\widehat{y}}\,+\,\rho)^3}\\
&\leq &
\frac{1}{\sqrt{n}}\,\,
\frac{C_{u,u}}{\sum_{\widehat{y} \ne u} C_{u,\widehat{y}}\,+\,\rho}\,\,\frac{1}{(\sum_{\widehat{y} = 1}^n C_{v,\widehat{y}}\,+\,\rho)^2}\\
&\leq &
\frac{1}{\sqrt{n}}\,
\frac{1}{\rho^3}.
\end{eqnarray*}
\noindent The smoothness parameter $\beta_\rho$ in \Tab{tab:perf-metrics-smoothed} then follows from the above bounds.



\noindent \textbf{G-mean}. For the G-mean performance metric, $\psi^\text{G}(\C) \,=\,  \Big(\prod_{y=1}^n \frac{C_{y,y}}{\sum_{\widehat{y} = 1}^n C_{y,\widehat{y}}}\Big)^{1/n}$ is not Lipschitz over $\CC_D$. We now explicitly derive the form of $\theta$ for this performance metric.
\begin{eqnarray*}
\psi^\text{G}_\rho (\C) \,-\, \psi^\text{G}(\C)
&=& \bigg(\prod_{y=1}^n \frac{C_{y,y} + \rho}{\sum_{\widehat{y} = 1}^n C_{y,\widehat{y}}+\rho}\bigg)^{1/n}
\,-\, \bigg(\prod_{y=1}^n \frac{C_{y,y}}{\sum_{\widehat{y} = 1}^n C_{y,\widehat{y}}}\bigg)^{1/n}\\
&\leq& \bigg(\prod_{y=1}^n \bigg(\frac{C_{y,y}}{\sum_{\widehat{y} = 1}^n C_{y,\widehat{y}}}+\rho\bigg)\bigg)^{1/n}
\,-\, \bigg(\prod_{y=1}^n \frac{C_{y,y}}{\sum_{\widehat{y} = 1}^n C_{y,\widehat{y}}}\bigg)^{1/n}\\
&\leq& \bigg(\prod_{y=1}^n \frac{C_{y,y}}{\sum_{\widehat{y} = 1}^n C_{y,\widehat{y}}} \,+\, \sum_{y = 1}^n{n \choose y}\rho^{y}\bigg)^{1/n}
\,-\, \bigg(\prod_{y=1}^n \frac{C_{y,y}}{\sum_{\widehat{y} = 1}^n C_{y,\widehat{y}}}\bigg)^{1/n}\\
&\leq& \bigg(\prod_{y=1}^n \frac{C_{y,y}}{\sum_{\widehat{y} = 1}^n C_{y,\widehat{y}}} \,+\, \sum_{y = 1}^n{n \choose y}\rho\bigg)^{1/n}
\,-\, \bigg(\prod_{y=1}^n \frac{C_{y,y}}{\sum_{\widehat{y} = 1}^n C_{y,\widehat{y}}}\bigg)^{1/n}\\
&\leq& \bigg(\prod_{y=1}^n \frac{C_{y,y}}{\sum_{\widehat{y} = 1}^n C_{y,\widehat{y}}}+ 2^n\rho\bigg)^{1/n}
\,-\, \bigg(\prod_{y=1}^n
\frac{C_{y,y}}{\sum_{\widehat{y} = 1}^n C_{y,\widehat{y}}}\bigg)^{1/n}\\
&\leq& \bigg(\prod_{y=1}^n\frac{C_{y,y}}{\sum_{\widehat{y} = 1}^n C_{y,\widehat{y}}}\bigg)^{1/n} \,+\, 2\rho^{1/n} \,-\, \bigg(\prod_{y=1}^n\frac{C_{y,y}}{\sum_{\widehat{y} = 1}^n C_{y,\widehat{y}}}\bigg)^{1/n}\\
&=& 2\rho^{1/n},
\end{eqnarray*}
which gives us $\theta(\rho) \,\leq\, 2\rho^{1/n}$. Next, we provide the gradient $\psi^{\text{G}}_\rho$.
\[
\nabla_{C_{uu'}} \psi^\text{G}_\rho (\C) \,=\, 
\begin{cases}
\frac{1}{n}
\Big(\frac{C_{u,y}+\rho}{\sum_{\widehat{y} = 1}^n C_{u,\widehat{y}}+\rho}\Big)^{\frac{1}{n}-1}\frac{\sum_{\widehat{y} \ne u} C_{u,\widehat{y}}}{(\sum_{\widehat{y} = 1}^n C_{u,\widehat{y}}  + \rho)^2}\prod_{y \ne u}\Big(\frac{C_{y,y}+\rho}{\sum_{\widehat{y} = 1}^n C_{y,\widehat{y}} + \rho}\Big)^{1/n}
&\text{if $u = u'$}\\
-\frac{1}{n}
\Big(\frac{C_{u,y}+\rho}{\sum_{\widehat{y} = 1}^n C_{u,\widehat{y}}+\rho}\Big)^{\frac{1}{n}-1}\frac{C_{u,u}}{(\sum_{\widehat{y} = 1}^n C_{u,\widehat{y}}  + \rho)^2}\prod_{y \ne u}\Big(\frac{C_{y,y}+\rho}{\sum_{\widehat{y} = 1}^n C_{y,\widehat{y}} + \rho}\Big)^{1/n}
&\text{otherwise}
\end{cases}.
\]

\noindent The Lipschitz constant for $\psi^{\text{G}}_\rho$ is then obtained by bounding the norm of the above gradient.
\begin{eqnarray*}
\|\nabla \psi^\text{G}_\rho (\C)\|_\infty 
&\leq& 
\max_{u \,\in\, [n]}\, \frac{1}{n}
\bigg(\frac{C_{u,u}+\rho}{\sum_{\widehat{y} = 1}^n C_{u,\widehat{y}}+\rho}\bigg)^{\frac{1}{n}-1}\,
\frac{\max\{\sum_{\widehat{y} \ne u} C_{u,\widehat{y}}, \, C_{u,u}\}}{(\sum_{\widehat{y} = 1}^n C_{u,\widehat{y}}  + \rho)^2}\\
&\leq& 
\max_{u \,\in\, [n]}\, \frac{1}{n}
\bigg(\frac{C_{u,u}+\rho}{\sum_{\widehat{y} = 1}^n C_{u,\widehat{y}}+\rho}\bigg)^{\frac{1}{n}-1}
\frac{\sum_{\widehat{y} = 1}^n C_{u,\widehat{y}} + \rho}{(\sum_{\widehat{y} = 1}^n C_{u,\widehat{y}}  + \rho)^2}\\
&=& 
\max_{u \,\in\, [n]}\, \frac{1}{n}
\bigg(\frac{C_{u,u}+\rho}{\sum_{\widehat{y} = 1}^n C_{u,\widehat{y}}+\rho}\bigg)^{\frac{1}{n}-1}
\,\frac{1}{\sum_{\widehat{y} = 1}^n C_{u,\widehat{y}}  + \rho}\\
& \leq & 
\frac{1}{n} \bigg(\frac{\rho}{1+\rho}\bigg)^{\frac{1}{n}-1}\frac{1}{\rho}\\
& = & 
\frac{1}{n} \bigg(\frac{1+\rho}{\rho}\bigg)^{1-\frac{1}{n}}\frac{1}{\rho}.
\end{eqnarray*}

\noindent The Hessian for $\psi^\text{G}_\rho$ takes the form:
\begin{eqnarray*}
\lefteqn{\nabla^2_{C_{uu'},C_{vv'}} \psi^\text{G}_\rho(\C) \,\,=\,}\\
&
\begin{cases}
-\frac{1}{n}\Big(1-\frac{1}{n}\Big)
\Big(\frac{C_{u,u}+\rho}{\sum_{\widehat{y} = 1}^n C_{u,\widehat{y}}+\rho}\Big)^{\frac{1}{n}-2}
\Big(\frac{\sum_{\widehat{y} \ne u} C_{u,\widehat{y}}}{(\sum_{\widehat{y} = 1}^n C_{u,\widehat{y}}  + \rho)^2}\Big)^2\prod_{y \ne u}\Big(\frac{C_{y,y}+\rho}{\sum_{\widehat{y} = 1}^n C_{y,\widehat{y}} + \rho}\Big)^{1/n}
\\
\hspace{7cm} \text{if $u' = u = v = v'$}\\
\frac{1}{n}\Big(1-\frac{1}{n}\Big)
\Big(\frac{C_{u,u}+\rho}{\sum_{\widehat{y} = 1}^n C_{u,\widehat{y}}+\rho}\Big)^{\frac{1}{n}-2}
\frac{C_{u,u}\sum_{\widehat{y} \ne u} C_{u,\widehat{y}}}{(\sum_{\widehat{y} = 1}^n C_{u,\widehat{y}}  + \rho)^4}
\prod_{y \ne u}\Big(\frac{C_{y,y}+\rho}{\sum_{\widehat{y} = 1}^n C_{y,\widehat{y}} + \rho}\Big)^{1/n}
\\
\hspace{7cm}\text{if\, $u' = u = v \ne v'$ \,or\, $u' \ne u = v = v'$ \,or\, $u' \ne u = v \ne v'$}
\\
\frac{1}{n^2}
\Big(\frac{C_{u,u}+\rho}{\sum_{\widehat{y} = 1}^n C_{u,\widehat{y}}+\rho}\Big)^{\frac{1}{n}-1}
\Big(\frac{C_{v,v}+\rho}{\sum_{\widehat{y} = 1}^n C_{v,\widehat{y}}+\rho}\Big)^{\frac{1}{n}-1}
\frac{\sum_{\widehat{y} \ne u} C_{u,\widehat{y}}}{(\sum_{\widehat{y} = 1}^n C_{u,\widehat{y}}  + \rho)^2}
\frac{\sum_{\widehat{y} \ne v} C_{v,\widehat{y}}}{(\sum_{\widehat{y} = 1}^n C_{v,\widehat{y}}  + \rho)^2}
\underset{y \ne u, y \ne v}{\prod}\Big(\frac{C_{y,y} + \rho}{\sum_{\widehat{y} = 1}^n C_{y,\widehat{y}} + \rho}\Big)^{1/n}
\\\hspace{7cm}\text{if $u' = u \ne v = v'$}
\\
-\frac{1}{n^2}
\Big(\frac{C_{u,u}+\rho}{\sum_{\widehat{y} = 1}^n C_{u,\widehat{y}}+\rho}\Big)^{\frac{1}{n}-1}
\Big(\frac{C_{v,v}+\rho}{\sum_{\widehat{y} = 1}^n C_{v,\widehat{y}}+\rho}\Big)^{\frac{1}{n}-1}
\frac{\sum_{\widehat{y} \ne u} C_{u,\widehat{y}}}{(\sum_{\widehat{y} = 1}^n C_{u,\widehat{y}}  + \rho)^2}
\frac{C_{v,v}}{(\sum_{\widehat{y} = 1}^n C_{v,\widehat{y}}  + \rho)^2}
\underset{y \ne u, y \ne v}{\prod}\Big(\frac{C_{y,y} + \rho}{\sum_{\widehat{y} = 1}^n C_{y,\widehat{y}} + \rho}\Big)^{1/n}
\\\hspace{7cm}\text{if $u' = u \ne v \ne v'$}
\\
-\frac{1}{n^2}
\Big(\frac{C_{u,u}+\rho}{\sum_{\widehat{y} = 1}^n C_{u,\widehat{y}}+\rho}\Big)^{\frac{1}{n}-1}
\Big(\frac{C_{v,v}+\rho}{\sum_{\widehat{y} = 1}^n C_{v,\widehat{y}}+\rho}\Big)^{\frac{1}{n}-1}
\frac{C_{u,u}}{(\sum_{\widehat{y} = 1}^n C_{u,\widehat{y}}  + \rho)^2}
\frac{\sum_{\widehat{y} \ne v} C_{v,\widehat{y}}}{(\sum_{\widehat{y} = 1}^n C_{v,\widehat{y}}  + \rho)^2}
\underset{y \ne u, y \ne v}{\prod}\Big(\frac{C_{y,y} + \rho}{\sum_{\widehat{y} = 1}^n C_{y,\widehat{y}} + \rho}\Big)^{1/n}
\\\hspace{7cm}\text{if $u' \ne u \ne v = v'$}
\end{cases}\\
&
\hspace{-1.8cm}
\begin{cases}
\frac{1}{n^2}
\Big(\frac{C_{u,u}+\rho}{\sum_{\widehat{y} = 1}^n C_{u,\widehat{y}}+\rho}\Big)^{\frac{1}{n}-1}
\Big(\frac{C_{v,v}+\rho}{\sum_{\widehat{y} = 1}^n C_{v,\widehat{y}}+\rho}\Big)^{\frac{1}{n}-1}
\frac{C_{u,u}}{(\sum_{\widehat{y} = 1}^n C_{u,\widehat{y}}  + \rho)^2}
\frac{C_{v,v}}{(\sum_{\widehat{y} = 1}^n C_{v,\widehat{y}}  + \rho)^2}
\underset{y \ne u, y \ne v}{\prod}\Big(\frac{C_{y,y} + \rho}{\sum_{\widehat{y} = 1}^n C_{y,\widehat{y}} + \rho}\Big)^{1/n}
\\\hspace{7cm}\text{otherwise}.
\end{cases}
\end{eqnarray*}

\noindent The smoothness parameter $\beta_\rho$ is then given by the following bound on the norm of the Hessian:
\begin{eqnarray*}
\lefteqn{\|\nabla^2 \psi^\text{G}_\rho(\C)\|_\infty}\\
&\leq& \max_{u, v \,\in\, [n]}\, 
\frac{1}{n^2}
\Big(\frac{C_{u,u}+\rho}{\sum_{\widehat{y} = 1}^n C_{u,\widehat{y}}+\rho}\Big)^{\frac{1}{n}-1}
\Big(\frac{C_{v,v}+\rho}{\sum_{\widehat{y} = 1}^n C_{v,\widehat{y}}+\rho}\Big)^{\frac{1}{n}-1}
\,\frac{\max\{C_{u,u},\,\sum_{\widehat{y} \ne u} C_{u,\widehat{y}}\}}{(\sum_{\widehat{y} = 1}^n C_{u,\widehat{y}}  + \rho)^2}
\frac{\max\{C_{v,v},\,\sum_{\widehat{y} \ne v} C_{v,\widehat{y}}\}}{(\sum_{\widehat{y} = 1}^n C_{v,\widehat{y}}  + \rho)^2}\\
&\leq& 
\max_{u \,\in\, [n]}\, \frac{1}{n^2}
\bigg(\frac{C_{u,u}+\rho}{\sum_{\widehat{y} = 1}^n C_{u,\widehat{y}}+\rho}\bigg)^{\frac{2}{n}-2}
\bigg(\frac{\sum_{\widehat{y} = 1}^n C_{u,\widehat{y}} + \rho}{(\sum_{\widehat{y} = 1}^n C_{u,\widehat{y}}  + \rho)^2}\bigg)^2\\
&=& 
\max_{u \,\in\, [n]}\, \frac{1}{n^2}
\bigg(\frac{C_{u,u}+\rho}{\sum_{\widehat{y} = 1}^n C_{u,\widehat{y}}+\rho}\bigg)^{\frac{2}{n}-2}
\bigg(\frac{1}{\sum_{\widehat{y} = 1}^n C_{u,\widehat{y}}  + \rho}\bigg)^2\\
&\leq& 
\frac{1}{n^2}
\bigg(\frac{\rho}{1+\rho}\bigg)^{\frac{2}{n}-2}
\frac{1}{\rho^2}\\
&=& 
\frac{1}{n^2}
\bigg(\frac{1+\rho}{\rho}\bigg)^{2-\frac{2}{n}}
\frac{1}{\rho^2}.
\end{eqnarray*}

\end{document}